%% file: sgd-tpr-tnr.tex
\documentclass{article}

\usepackage{amsmath,amsthm,amssymb}

\usepackage{times}
\usepackage{graphicx} % more modern
\usepackage{subfigure}

\usepackage{fullpage}

\usepackage{natbib}

\usepackage{algorithm}
\usepackage{algorithmic}

\usepackage{hyperref}
\usepackage{xspace}
\usepackage{dblfloatfix}% http://ctan.org/pkg/dblfloatfix
\usepackage{lipsum}% http://ctan.org/pkg/lipsum

\usepackage{epstopdf}

\usepackage{stmaryrd,dsfont}

\usepackage{pifont}
\newcommand{\xmark}{\ding{55}}%

\include{defs}

\title{Optimizing Non-decomposable Performance Measures: A Tale of Two Classes}
\author{Harikrishna Narasimhan$^\ast$\\Indian Institute of Science\\harikrishna@csa.iisc.ernet.in \and Purushottam Kar\\Microsoft Research India\\t-purkar@microsoft.com \and Prateek Jain\\Microsoft Research India\\prajain@microsoft.com}

\begin{document}

\maketitle

\begin{abstract}
Modern classification problems frequently present mild to severe label imbalance as well as specific requirements on classification characteristics, and require optimizing performance measures that are non-decomposable over the dataset, such as F-measure. Such measures have spurred much interest and pose specific challenges to learning algorithms since their non-additive nature precludes a direct application of well-studied large scale optimization methods such as stochastic gradient descent.

In this paper we reveal that for two large families of performance measures that can be expressed as functions of true positive/negative rates, it is indeed possible to implement point stochastic updates. The families we consider are concave and pseudo-linear functions of TPR, TNR which cover several popularly used performance measures such as F-measure, G-mean and H-mean.

Our core contribution is an adaptive linearization scheme for these families, using which we develop optimization techniques that enable truly point-based stochastic updates. For concave performance measures we propose \pdsgd, a stochastic primal dual solver; for pseudo-linear measures we propose \amsgd, a stochastic alternate maximization procedure. Both methods have crisp convergence guarantees, demonstrate significant speedups over existing methods - often by an order of magnitude or more, and give similar or more accurate predictions on test data.

\end{abstract}

\let\thefootnote\relax\footnotetext{$^\ast$ Work done while H.N. was an intern at Microsoft Research India, Bangalore.}

\input{intro}
\input{formulation}
\input{concave}
\input{pseudolinear}
\input{expts}

\section*{Acknowledgements}
HN thanks support from a Google India PhD Fellowship.

\bibliographystyle{plainnat}
\bibliography{refs}

\appendix
\onecolumn
\allowdisplaybreaks

\input{app}

\end{document}

%% file: defs.tex
\renewcommand{\>}{\rightarrow}
\newcommand{\<}{\leftarrow}

\newcommand{\ind}[1]{\mathbf{1}\br{#1}}

\newcommand{\B}{{\cal B}}
\newcommand{\R}{{\mathbb R}}

\newcommand{\T}{{\cal T}}

\newcommand{\pdsgd}{{\bfseries SPADE}\xspace}
\newcommand{\am}{{\bfseries AMP}\xspace}
\newcommand{\amsgd}{{\bfseries STAMP}\xspace}
\newcommand{\spmb}{{\bfseries 1PMB}\xspace}

\newcommand{\1}{{\mathbf 1}}

\newcommand{\X}{{\mathcal X}}
\newcommand{\Y}{{\mathcal Y}}

\newcommand{\btheta}{{\boldsymbol \theta}}

\newcommand{\A}{{\mathcal A}}

\newcommand{\D}{\mathcal D}

\newcommand{\W}{{\mathcal W}}

\newcommand{\barw}{{\overline \w}}

\newcommand{\Pf}{{\mathcal P}}

\renewcommand{\vec}[1]{{\mathbf{#1}}}

\newcommand{\veczero}{\vec{0}}
\newcommand{\vz}{\veczero}

\newcommand{\x}{\vec{x}}

\newcommand{\w}{\vec{w}}
\renewcommand{\v}{\vec{v}}

\newcommand{\ba}{\vec{a}}
\newcommand{\bb}{\vec{b}}

\newcommand{\bc}[1]{\left\{{#1}\right\}}
\newcommand{\br}[1]{\left({#1}\right)}
\newcommand{\bs}[1]{\left[{#1}\right]}
\newcommand{\abs}[1]{\left| {#1} \right|}
\newcommand{\norm}[1]{\left\| {#1} \right\|}

\newcommand{\bsd}[1]{\left\llbracket{#1}\right\rrbracket}

\renewcommand{\O}[1]{{\cal O}\br{{#1}}}
\newcommand{\softO}[1]{\widetilde{\cal O}\br{{#1}}}

\newcommand{\Eemp}[2]{{{\mathbb E}}_{#1}\bsd{{#2}}}

\newcommand{\EE}[2]{\underset{#1}{\mathbb E}\bsd{{#2}}}

\newcommand{\Prr}[2]{\underset{#1}{\text{Pr}}\bs{{#2}}}
\newcommand{\ip}[2]{\left\langle{#1},{#2}\right\rangle}

\newtheorem{lem}{Lemma}
\newtheorem{thm}[lem]{Theorem}

\newtheorem{defn}[lem]{Definition}

\makeatletter
\newcommand{\newreptheorem}[2]{\newtheorem*{rep@#1}{\rep@title} 
\newenvironment{rep#1}[1]{\def\rep@title{#2 \ref*{##1}}\begin{rep@#1}}{\end{rep@#1}}
}
\makeatother

\newreptheorem{lem}{Lemma}
\newreptheorem{thm}{Theorem}
\newreptheorem{clm}{Claim}

%% file: intro.tex
\section{Introduction}
\label{sec:intro}
Learning applications with binary classification problems involving severe label imbalance abound, often accompanied with specific requirements in terms of false positive, or negative rates. Examples included spam classification, anomaly detection, and medical applications. Class imbalance is also often introduced as a result of the reduction of a problem to binary classification, such as in multi-class problems \cite{Bishop06} and multi-label problems due to extreme label sparsity \cite{HsuKLZ09}.

Traditional performance measures such as misclassification rate are ill-suited in such situations as it is usually trivial to optimize them by constantly predicting the majority class. Instead, the performance measures of choice in such cases are those that perform a more holistic evaluation over the entire data. Naturally, these performance measures are {\em non-decomposable} over the dataset and cannot be cannot be expressed as a sum of errors on individual data points. Popular examples include F-measure, G-mean, H-mean etc.

A consistent effort directed at optimizing these performance measures has, over the years, resulted in the development of two broad approaches - 1) surrogate based approaches (e.g. SVMPerf \cite{JoachimsFY09}) that design convex surrogates for these performance measures, and 2) indirect approaches which include cost-sensitive classification-based approaches \cite{ParambathUG14} which solve weighted classification problems, and plug-in approaches \cite{KoyejoNRD14,NarasimhanVA14} which rely on consistent estimates of class probabilities.

Both these approaches are known to work fairly well on small datasets but do not scale very well to large ones, especially those large enough to not even fit in memory. SVMPerf-style approaches, which employ cutting plane methods do not scale well. On the other hand, plug-in approaches first need to solve a class probability estimation problem optimally and then tune a threshold. This two-stage approach prevents the method from exploiting better classifiers to automatically obtain better thresholds. Moreover, for multi-class problems with $C$ classes, jointly estimating $C$ parameters can take time exponential in $C$.

For large datasets, streaming methods such as stochastic gradient descent \cite{ShwartzSSC11} that take only a few passes over the entire data are preferable. However, traditional SGD techniques cannot handle non-decomposable losses. Recently, \cite{KarNJ14} proposed optimizing SVMPerf-style surrogates using SGD techniques. Although their method is generic, allowing optimization of performance measures such as F-measure and partial AUC, they require maintaining a \emph{large buffer} to compute online gradient estimates that can be prohibitive.

Motivated by the state of the art, we develop novel methods for optimizing two broad families of non-decomposable performance measures. Our methods incorporate truly point-wise updates, i.e. do not require a buffer, and require only a few passes over data. At an intuitive level, at the core of our work are adaptive linearization strategies for these performance measures, which make these measures amenable to SGD-style point-wise updates. Moreover, our linearizations are able to feed off the improvements made in learning a better classifier, resulting in faster convergence.

We consider two classes of performance measures

{\bf Concave Performance Measures} (see Table~\ref{tab:concave-perf-list}): These measures can be written as concave functions of true positive (TPR) and negative (TNR) rates and include G-mean, H-mean etc. We exploit the dual structure of these functions via their Fenchel dual to linearize them in terms of the TPR, TNR variables. Our method then, in parallel, tunes the dual variables in this linearization and maximizes the weighted TPR-TNR combination. These updates are done in an online fashion using stochastic mirror descent steps.

{\bf Pseudo-linear Performance Measures} (see Table~\ref{tab:pseudolinear-perf-list}): These measures can be written as fractional linear functions of TPR, TNR and include F-measure and the Jaccard coefficient. These functions need not be concave and the techniques outlined above do not apply. Instead, we exploit the pseudo-linear structure to linearize the function and develop a technique to alternately optimize the combination weights and the classifier model via stochastic updates. Although such ``alternate-maximization'' strategies in general need not converge even to a local optima, we show that our strategy converges to an $\epsilon$-approximate global optimum after $\log\br{\frac{1}{\epsilon}}$ batch updates or $O(1/\epsilon^2)$ stochastic updates.

Finally, we present an empirical validation of our methods. Our experiments reveal that for a range of performance measures in both classes, our methods can be significantly faster than either plug-in or SVMPerf-style methods, as well as give higher or comparable accuracies.

\section{Related Works}
\label{sec:rel}
As noted in Section~\ref{sec:intro}, existing methods for optimizing performance measures that we study can be divided into surrogate-based approaches and indirect approaches based on cost-sensitive classification or plug-in methods. A third approach applicable to certain performance measures is the \emph{decision-theoretic} method that learns a class probability estimate and computes predictions that maximize the expected value of the performance measure on a test set \cite{Lewis95, YeCLC12}. In addition to these there exist methods dedicated to specific performance measures.

For instance \cite{ParambathUG14} focus on optimizing F-measure by exploiting the pseudo-linearity of the function along with a cross validation-based strategy. Our \amsgd method, on the other hand uses an alternating maximization strategy that does not require cross validation which considerably improves training time (see Figure~\ref{fig:F1}). It is important to note that these performance measures have also been studied in multi-label settings where these no longer remain non-decomposable. For instance, \cite{DembczynskiJKWH13} study plug-in style methods for maximizing F-measure in multi-label settings whereas works such as \cite{KoyejoNRD14,NarasimhanVA14,YeCLC12} study plug-in approaches for the same problem in the more challenging binary classification setting.

Historically, online learning algorithms have played a key role in designing solvers for large-scale batch problems. However, for non-decomposable loss functions, defining an online learning framework and providing efficient algorithms with small regret itself is challenging. \cite{RakhlinST11} propose a generic method for such loss functions; however the algorithms proposed therein run in exponential time. \cite{KarNJ14} also study such measures with the aim of designing stochastic gradient-style methods. However, their methods require a large buffer to be maintained, which causes them to have poorer convergence guarantees and in practice be slower than our methods. 

By exploiting the special structure in our function classes, we are able to do away with such requirements. Our methods make use of standard online convex optimization primitives \cite{zinkevich}. However, their application requires special care in order to avoid divergent behavior.

%%% Local Variables: 
%%% mode: latex
%%% TeX-master: "sgd-tpr-tnr"
%%% End: 

%% file: formulation.tex
\section{Problem Setting}
\label{sec:formulation}
Let $\X \subset \R^d$ denote the instance space and $\Y = \bc{-1,+1}$ the label space, with some distribution $\D$ over $\X\times\Y$. Let $p := \Prr{(\x,y)\sim\D}{y = +1}$ denote the proportion of positives in the population. Let $\T = \bc{(\x_1,y_1),\ldots,(\x_T,y_T)}$ denote a sample of training points sampled i.i.d. from $\D$. For sake of simplicity we shall present our algorithms and analyses for a set of linear models $\W \subseteq \R^d$. Let $R_\X$ and $R_\W$ denote the radii of the domain $\X$ and hypothesis class $\W$ respectively.

We consider performance measures that can be expressed in terms of the true positive and negative rates of a classifier. To represent these measures, we shall use the notion of a \emph{reward function} $r$ that assigns a \emph{reward} $r(y, \hat{y})$ to a prediction $\hat{y} \in \R$ when the true label is $y \in \Y$. We will use
\begin{align*}
r^+(\w; \x, y) &= \frac{1}{p}\cdot r(y,\w^\top\x)\cdot\1(y = 1)\\
r^-(\w;\x, y) &= \frac{1}{1-p}\cdot r(y,\w^\top\x)\cdot\1(y=-1)
\end{align*}
to calculate rewards on positive and negative points. Since $\EE{(\x,y)}{r^+(\w; \,\x, y)} = \EE{(\x,y)}{r(y,\w^\top\x)|y = 1}$, setting $r^{0-1}(y, \hat{y}) = \ind{y\hat{y} > 0}$ gives us $\EE{(\x,y)}{r^+(\w; \,\x, y)} = \text{TPR}(\w)$. For sake of convenience, we will use $P(\w)=\EE{(\x, y)}{r^+(\w; \x, y)}$ and $N(\w)=\EE{(\x, y)}{r^-(\w; \x, y)}$ to denote population averages of the reward functions. We shall assume that our reward functions are concave, $L_r$-Lipschitz, and take values in a bounded range $[-B_r,B_r]$.

%% file: concave.tex
\section{Concave Performance Measures}
\label{sec:concave}
The first class of performance measures we analyze are concave performance measures. These measures can be written as concave functions of the TPR and TNR i.e.
\[
\Pf_\Psi(\w) = \Psi\br{P(\w),N(\w)}
\]
for some concave link function $\Psi: \R^2 \> \R$. A large number of popular performance measures fall in this family since these measures are relevant in situations with severe label imbalance or in situations where cost-sensitive classification is required such as detection theory \cite{Vincent94}. Table~\ref{tab:concave-perf-list} gives a list of such performance measures along with some of their relevant properties and references to works that utilize these performance measures.
\begin{table*}[t]
\caption{List of concave performance measures $\Psi(P,N)$ along with their monotonicity and Lipschitz properties, sufficient dual regions, and expressions for dual subgradients. $\B(\vz,r)$ denotes the ball of radius $r$ around the origin. $\R_+^2$ denotes the positive quadrant.}
\centering
\small{
\begin{tabular}{|c|c|c|c|c|c|c|c|}
\hline
Name & Expression $(P,N)$ & Mon.? & Lip.? & $\delta(\epsilon)$ & Sufficient dual Region $\A_\Psi$ & $\nabla\Psi^\ast(\alpha,\beta)$.\\
\hline
\hline
Min ({\tiny \cite{Vincent94}}) & $\min\{P, N\}$ & \checkmark & \checkmark & $\epsilon$ & $\bc{\alpha+\beta=1} \cap \R_+^2$ & $\vz$ \\
\hline
H-mean ({\tiny \cite{KennedyND09}}) & $\frac{2PN}{P+N}$ & \checkmark & \checkmark & $4\epsilon$ & $\bc{\sqrt{\alpha} + \sqrt{\beta} \geq \sqrt 2} \cap \B(\vz,2)$ & $\vz$ \\
\hline
Q-mean ({\tiny \cite{LiuCh11}}) & $1-\sqrt{\frac{(1-P)^2 + (1-N)^2}{2}}$ & \checkmark & \checkmark & $\epsilon$ & $\bc{\alpha^2 + \beta^2 \leq 1/2}\cap \R_+^2$ & $\mathbf{1}$\\
\hline
G-mean ({\tiny \cite{DaskalakiKA06}}) & $\sqrt{PN}$ & \checkmark & \xmark & $3\sqrt\epsilon$ & $\bc{\alpha\beta \geq 1/4} \cap \R_+^2$ & $\vz$ \\
\hline
\end{tabular}
}
\label{tab:concave-perf-list}
\end{table*}

We shall find it convenient to define the (concave) Fenchel conjugate of the link functions for our performance measures. For any concave function $\Psi$ and $\alpha, \beta \in \R$, define
\[
\Psi^*(\alpha, \beta) = \inf_{u, v \in \R} \bc{\alpha u + \beta v  - \Psi(u, v)}.
\]
By the concavity of $\Psi$, we have, for any $u, v \in \R$,
\[
\Psi(u, v) = \inf_{\alpha, \beta \in \R} \bc{\alpha u + \beta v  - \Psi^*(\alpha, \beta)}.
\]
We shall use the notation $\Psi$ to denote, both the link function, as well as the performance measure it induces.

\input{concave-method}
\input{concave-analysis}

\input{concave-nonLipschitz}

%% file: concave-method.tex
\subsection{A Stochastic Primal-dual Method for Optimizing Concave Performance Measures}
We now present a novel online stochastic method for optimizing the class of concave performance measures. The

\begin{algorithm}[t]
	\caption{\small \pdsgd: Stochastic PrimAl-Dual mEthod}
	\label{algo:spdu}
	\begin{algorithmic}[1]
		\small{
			\REQUIRE Primal/dual step sizes $\eta_t, \eta_t'$, feasible sets $\W, \A$
			\ENSURE Classifier $\w \in \W$
			\STATE $\w_0 \< \vz, t \< 1$
			\WHILE{data stream has points}
				\STATE Receive data point $(\x_t,y_t)$
				\STATE \COMMENT{Perform primal ascent}
				\IF{$y_t > 0$}
					\STATE $\w_{t+1} \< \Pi_\W\br{\w_t + \eta_t\cdot\alpha_t\nabla_\w r^+(\w_t; \x_t, y_t)}$
				\ELSE
					\STATE $\w_{t+1} \< \Pi_\W\br{\w_t + \eta_t\cdot\beta_t\nabla_\w r^-(\w_t; \x_t, y_t)}$
				\ENDIF
				\STATE \COMMENT{Perform dual descent}
				\STATE $(a,b) \< (\alpha_t, \beta_t) - \eta_t'\cdot\nabla_{(\alpha,\beta)}\Psi^*(\alpha_t, \beta_t)$
				\IF{$y_t > 0$}
					\STATE $a \< a - \eta_t'\cdot r^+(\w_t; \x_t, y_t)$
				\ELSE
					\STATE $b \< b - \eta_t'\cdot r^-(\w_t; \x_t, y_t)$
				\ENDIF
				\STATE $(\alpha_{t+1}, \beta_{t+1}) \< \Pi_{\A}((a,b))$
				\STATE $t \< t+1$
			\ENDWHILE
			\STATE \textbf{return} $\barw = \frac{1}{t}\sum_{\tau = 1}^{t}\w_\tau$
		}
	\end{algorithmic}
\end{algorithm}

use of stochastic gradient techniques for these measures presents specific challenges due to the non-decomposable nature of these measures which makes it difficult to obtain cheap, unbiased estimates of the gradient using a single point. Recent works \cite{KarSJK13,KarNJ14} have tried to resolve this issue by looking at mini-batch methods or by using a buffer to maintain a sketch of the stream. However, such techniques bring in a bias into the learning algorithm in the form of buffer size or mini batch length which results in slower convergence. Indeed, the \spmb method of \cite{KarNJ14} is only able to guarantee a $\sqrt[-4]{T}$ rate of convergence, whereas SGD techniques are usually able to guarantee $\sqrt[-2]{T}$ rates. This is indicative of suboptimal performance and our experiments confirm this (see Figure~\ref{fig:F1}).

Here we show that for the class of concave performance measures, such workarounds are not necessary. To this end we present the \pdsgd algorithm (Algorithm~\ref{algo:spdu}) which exploits the dual structure of the performance measures to obtain efficient point updates which do not require the use of mini-batches or online buffers. \pdsgd is able to offer convergence guarantees identical to those that stochastic methods offer for additive performance measures such as least squares, without the presence of any algorithmic bias.

Let $\W \subset \X$ and $\A_\Psi \subset \R^2$ be convex regions within the model and dual spaces respectively, and $\Pi_\W$ and $\Pi_{\A_\Psi}$ denote projection operators for these. Table~\ref{tab:concave-perf-list} lists the relevant dual regions for the performance measures listed therein.

%% file: concave-analysis.tex
\subsection{Convergence Analysis for \pdsgd}
This section presents a convergence analysis for the \pdsgd algorithm. The convergence proof is formally stated in Theorem~\ref{thm:pdsgd-risk-analysis}. Apart from demonstrating the utility of the algorithm, the proof also sheds light on the choice of algorithm parameters, such as primal/dual feasible regions.

We shall work with performance measures that are monotonically increasing in the true positive and negative rates of the classifier i.e. if $u \geq u'$, $v \geq v'$ then $\Psi(u,v) \geq \Psi(u',v')$. This is a natural assumption and is satisfied by all performance measures considered here (see Table~\ref{tab:concave-perf-list}). We now introduce two useful concepts.
\begin{defn}[Stable Performance Measure]
A performance measure $\Psi$ will be called $\delta$-\emph{stable} if for some function $\delta: \R \> \R$, we have for all $u, v \in \R$ and $\epsilon \in \R_+$,
\[
\Psi\br{u + \epsilon, v+ \epsilon} \leq \Psi(u, v) + \delta(\epsilon).
\]
\end{defn}
Table~\ref{tab:concave-perf-list} lists the stability parameters of all the concave performance measures. Clearly, a performance measure has a linear stability parameter i.e. $\delta(\epsilon) \leq L\cdot\epsilon$ iff its corresponding link function is Lipschitz. We now define the notion of a \emph{sufficient dual region} for a performance measure
\begin{defn}[Sufficient Dual Region]
For any link function $\Psi$, define its \emph{sufficient dual region} $\A_\Psi \subseteq \R^2$ to be the minimal set such that for all $(u,v) \in R^2$, we have
\[
\Psi(u, v) = \inf_{(\alpha, \beta) \in \A_\Psi} \bc{\alpha u + \beta v  - \Psi^*(\alpha, \beta)}.
\]
\end{defn}
The reason for defining this quantity will become clear in a moment. A closer look at Algorithm~\ref{algo:spdu} indicates that it is performing online gradient descent steps with the dual variables. Clearly, for this procedure to have statistical convergence properties, the magnitude of the updates must be bounded in some sense otherwise the learning procedure may diverge. This motivates the projection step in Step 17. However, in order for the updated dual variables to be informative about the current primal function value, the projection step must be done in a way that does not distort the link function. The notion of a sufficient dual region formally captures the notion of such a projection step.

Having said that, there is no apriori guarantee that the sufficient region for a given performance measure would be bounded, in which case this entire exercise counts for naught. However, the following lemma, by closely linking the stability properties of a performance measure with the size of its sufficient dual region, shows that for well-behaved link functions, this will not be the case .
\begin{lem}
\label{lem:dsr-stab}
The stability parameter of a performance measure $\Psi(\cdot)$ can be written as $\delta(\epsilon) \leq L_\Psi\cdot\epsilon$ iff its sufficient dual region is bounded in a ball of radius $L_\Psi$.
\end{lem}

The proof of this result follows from elementary manipulations and can be found in Appendix~\ref{app:lem-dsr-stab-proof}. In some sense this result can be seen as a realization of the well known connection between the Fenchel dual of a function and its Lipschitz properties.

To simplify the initial analysis, we shall first concentrate on performance measures whose link functions are Lipschitz. It is easy to see that these are exactly the performance measures whose gradients do not diverge within any compact region of the real plane. Of the performance measures listed in Table~\ref{tab:concave-perf-list}, all measures except G-mean have associated link functions that are Lipschitz. Subsequently, we shall address the more involved case of non-Lipschitz performance measures such as G-mean as well.

\begin{thm}
\label{thm:pdsgd-risk-analysis}
Suppose we are given a stream of random samples $(\x_1,y_1),\ldots,(\x_T,y_T)$ drawn from a distribution $\D$ over $\X\times\Y$. Let $\Psi(\cdot)$ be a concave, Lipschitz link function. Let Algorithm~\ref{algo:spdu} be executed with a dual feasible set $\A \supseteq \A_\Psi$, $\eta_t = 1/\sqrt t$ and $\eta'_t = 1/\sqrt t$. Then, the average model $\barw = \frac{1}{T}\sum_{t=1}^T \w_t$ output by the algorithm satisfies, with probability at least $1 - \delta$,
\[
\Pf_\Psi(\barw) \geq \sup_{\w^\ast \in \W}\Pf_\Psi(\w^\ast) -  \O{\delta_\Psi\br{\sqrt{\frac{1}{T}\log\frac{1}{\delta}}}}.
\]
\end{thm}
We refer the reader to Appendix~\ref{app:thm-pdsgd-risk-analysis-proof} for a proof and explicit constants. The proof closely analyzes the primal ascent and dual descent steps, tying them together using the Fenchel dual of $\Psi$.

%% file: concave-nonLipschitz.tex
\subsection{The Case of non-Lipschitz Link Functions}
Non-Lipschitz link functions, such as the one used in the G-mean performance measure, pose a particular challenge to the previous analysis. Owing to their non-Lipschitz nature, their sufficient dual region is unbounded. Indeed as Table~\ref{tab:concave-perf-list} indicates, the sufficient region for $\Psi_{\text{G-mean}}$ extends indefinitely along both coordinate axes. More precisely, what happens is that the gradients of the $\Psi_{\text{G-mean}}$ function diverge as either $u \rightarrow 0$, or $v \rightarrow 0$. This poses a stumbling block for the proof of Theorem~\ref{thm:pdsgd-risk-analysis} since the regret and online-to-batch conversion results used therein fail. 

A natural way to solve this problem is to ensure that the reward functions $r^+, r^-$ always assign rewards that are bounded away from zero. More specifically, for some $\epsilon >0$, we have $r^+(\w;\x,y), r^-(\w;\x,y) \geq \epsilon$ for all $\w \in \W$ and $\x \in \X$ . For this restricted reward region, one can show, using Lemma~\ref{lem:dsr-stab}, that the sufficient dual region can be restricted to a ball of radius $\O{\sqrt{1/\epsilon}}$.

The above discussion suggests that we regularize the reward function i.e. at each time step $t$, we add a small value $\epsilon(t)$ to the original reward function. However, the amount of regularization remains to be decided since over regularization could cause our resulting excess risk bound to be vacuous with respect to the original reward function. It turns out that setting $\epsilon(t) \approx \frac{1}{t^{1/4}}$ strikes a fine balance between regularization and fidelity to the original reward function - this seems intuitive since the regularization becomes milder and milder as learning progresses. The following extension of Theorem~\ref{thm:pdsgd-risk-analysis} formalizes this statement.

\begin{thm}
\label{thm:pdsgd-risk-analysis-non-lip}
Suppose we have the problem setting in Theorem~\ref{thm:pdsgd-risk-analysis} with the $\Psi_{\text{G-mean}}$ performance measure being optimized for. Consider a modification to Algorithm~\ref{algo:spdu} wherein the reward functions are changed to $r^+_t(\cdot) = r^+(\cdot) + \epsilon(t)$, and $r^-_t(\cdot) = r^-(\cdot) + \epsilon(t)$ for $\epsilon(t) = \frac{1}{t^{1/4}}$. Then, the average model $\barw = \frac{1}{T}\sum_{t=1}^T \w_t$ output by the algorithm satisfies, with probability at least $1 - \delta$,
\[
\Pf_{\Psi_{\text{G-mean}}}(\barw) \geq \sup_{\w^\ast \in \W}\Pf_{\Psi_{\text{G-mean}}}(\w^\ast) - \softO{\frac{1}{T^{1/4}}}.
\]
\end{thm}
The proof of this theorem can be found in Appendix~\ref{thm-pdsgd-risk-analysis-non-lip-proof}. We note here that primal dual frameworks have been utilized before in diverse areas such as distributed optimization \cite{JaggiSTTKJ14} and multi-objective optimization \cite{MahdaviYJ13}. However, these works simply assume the functions involved therein to be Lipschitz and/or smooth and do not address cases where they fail to be so. Theorem~\ref{thm:pdsgd-risk-analysis-non-lip} on the other hand, is able to recover a non-trivial, albeit weaker, statement even for \emph{locally Lipschitz} functions.

%% file: pseudolinear.tex
\section{Pseudo-linear Performance Measures}
\label{sec:pseudo-linear}
\begin{table*}[t]%
\caption{List of pseudo-linear performance measures $\Pf_{(\ba,\bb)}(P,N)$ along with their popular forms, canonical expressions in terms of (reward functions representative of) true positive (P) and negative (N) rates, monotonicity properties, acceptable range of reward values, and rate of convergence of the Alt-Max procedure for the performance measure when rewards take values in the range $[0,m)$.}
\centering
\small
\begin{tabular}{|c|c|c|c|c|c|c|c|c|c|c|}
\hline
Name & Popular Form & Canonical Form $(P,N)$ & Mon.? & Range $P,N$ & Rate $\eta(m)$\\
\hline
\hline
F$_\beta$-measure ({\tiny \citeauthor{Manning+08}}) & $\frac{(1 + \beta^2)\cdot TP}{(1 + \beta^2)\cdot TP + \beta^2\cdot TP + FP}$ & $\frac{(1 + \beta^2)\cdot P}{\beta^2 + \theta + P - \theta\cdot N}$ & \checkmark & $\br{0, 1+\frac{\beta^2}{\theta}}$  & $\frac{m(1+\theta)}{m + \beta^2 + \theta}$ \\\hline
Jaccard Coefficient ({\tiny \citeauthor{KoyejoNRD14}}) & $\frac{TP}{TP + FP + FN}$ & $\frac{P}{1 + \theta - \theta\cdot N}$ & \checkmark & $\br{0,\frac{1+\theta}{\theta}}$  & $\frac{m\theta}{1+\theta}$ \\\hline
Gower-Legendre$_{\sigma < 1}$ ({\tiny \citeauthor{SokolovaL09}}) & $\frac{TP + TN}{TP + \sigma(FP + FN) + TN}$ & $\frac{P + \theta\cdot N}{\sigma(1 + \theta) + (1-\sigma)\cdot P + \theta(1 - \sigma)\cdot N}$ & \checkmark & $(0, \infty)$  & $\frac{(1 - \sigma)m}{(1 - \sigma)m + \sigma}$ \\\hline
Gower-Legendre$_{\sigma > 1}$ ({\tiny \citeauthor{SokolovaL09}}) & $\frac{TP + TN}{TP + \sigma(FP + FN) + TN}$ & $\frac{P + \theta\cdot N}{\sigma(1 + \theta) + (1-\sigma)\cdot P + \theta(1 - \sigma)\cdot N}$ & \checkmark & $\br{0,\frac{\sigma}{\sigma-1}}$ & $\frac{(\sigma-1)m}{\sigma}$ \\\hline
\end{tabular}
\label{tab:pseudolinear-perf-list}
\end{table*}

The second class of performance measures we analyze are pseudo-linear performance measures. These measures have a fractional linear function as the link function and can be written as follows:
\[
\Pf_{(\ba,\bb)}(\w) = \frac{a_0 + a_1\cdot P(\w) + a_2\cdot N(\w)}{b_0 + b_1\cdot P(\w) + b_2\cdot N(\w)},
\]
 for some weighing coefficients $\ba,\bb$. Several popularly used performance measures, most notably the F-measure, can be represented as pseudo-linear functions. Table~\ref{tab:pseudolinear-perf-list} enumerates some popular pseudo-linear performance measures as well as their properties.

We note that these performance measures are usually represented in literature using the entries of the confusion matrix. However, for the sake of our analysis, we shall find it useful to represent them in terms of the true positive and true negative rates. To do so, we shall use $p$ to denote the proportion of positives in the population and $\theta = \frac{1-p}{p}$ to denote the label skew.

\input{pseudolinear-method}

%% file: pseudolinear-method.tex
\subsection{Alternate-maximization for Optimizing Pseudo-linear Performance Measures}
Pseudo-linear functions are named so since their level sets can be defined using linear half-spaces. More specifically, every pseudo-linear function $\Psi$ over $\R^d$ has an associated ``level-finder'' function $a: \R \> \R^d$ and $b: \R \> \R$ such that $\Psi(\v) \geq t$ iff $\ip{\v}{a(t)} \geq b(t)$. We refer the reader to \cite{ParambathUG14} for a more relaxed introduction to these functions and their properties. For our purposes, however, it suffices to notice that this property immediately points toward a cost-sensitive method to optimize these performance measures.

This fact was noticed by \cite{ParambathUG14} who exploited this to develop a cost-sensitive classification method for optimizing the F-measure by simply searching for the best weights with which to perform cost-sensitive classification. However, we notice that instead of performing such a brute force search, one can adaptively tune the weights to better and better values and obtain much faster convergence. To develop this intuition, we first define the notion of a \emph{valuation function} below.
\begin{defn}[Valuation Function]
The \emph{valuation function} of a performance measure $\Pf_{(\ba,\bb)}$, for a classifier $\w \in \W$, and at a level $v \in \R$ is defined as
\[
V_{(\ba,\bb)}(\w,v) := c + (\alpha - v\gamma) \cdot P(\w) + (\beta - v\delta) \cdot N(\w),
\]
where $c = \frac{a_0}{b_0}, \alpha = \frac{a_1}{b_0}, \beta = \frac{a_2}{b_0}, \gamma =\frac{b_1}{b_0}, \delta = \frac{b_2}{b_0}$.
\end{defn}
The following well-known lemma closely links the valuation function to the performance measure.
\begin{lem}
\label{lem:val-fn}
For any performance measure $\Pf_{(\ba,\bb)}$, $\w \in \W$ and $v\in\R$ we have $\Pf_{(\ba,\bb)}(\w) \geq v \text{ iff } V_{(\ba,\bb)}(\w,v) \geq v$. Moreover, in such a situation we say that classifier $\w$ has achieved valuation at \emph{level} $v$.
\end{lem}

Lemma~\ref{lem:val-fn} indicates that the performance of a classifier is intimately linked to its valuation. This suggests a natural alternate maximization approach wherein we alternate between posing a \emph{challenge level} to the classifier and training a classifier to achieve that level. The resulting algorithm \am is detailed in Algorithm~\ref{alg:am}. Note that using Lemma~\ref{lem:val-fn}, step 5 in the algorithm can be executed simply by setting $v_{t+1} = \Pf_{(\ba,\bb)}(\w_{t+1})$. Thus, in a very natural manner, the current classifier challenges the next classifier to beat its own performance. It turns out that this approach results in rapid convergence as outlined in the following theorem.

\begin{algorithm}[t]
	\caption{\small \am: Alternate Maximization Procedure}
	\label{alg:am}
	\begin{algorithmic}[1]
		\small{
			\REQUIRE Performance measure $\Pf_{(\ba,\bb)}$, feasible set $\W$, tolerance $\epsilon$
			\ENSURE An $\epsilon$-optimal classifier $\w \in \W$
			\STATE Construct valuation function $V_{(\ba,\bb)}$
			\STATE $\w_0 \< \vz, t \< 1$
			\WHILE{$v_t > v_{t-1} + \epsilon$}
				\STATE $\w_{t+1} \< \mathop{\arg\max}_{\w\in\W} V_{(\ba,\bb)}(\w,v_t)$
				\STATE $v_{t+1} \< \mathop{\arg\max}_{v > 0} v$ such that $V_{(\ba,\bb)}(\w_{t+1},v) \geq v$
				\STATE $t \< t+1$
			\ENDWHILE
			\STATE \textbf{return} $\w_t$
		}
	\end{algorithmic}
\end{algorithm}

\begin{algorithm}[t]
	\caption{\small \amsgd: STochastic Alt-Max Procedure}
	\label{alg:amsgd}
	\begin{algorithmic}[1]
		\small{
			\REQUIRE Feasible set $\W$, Step sizes $\eta_t$, epoch lengths $s_e, s'_e$
			\ENSURE Classifier $\w \in \W$
			\STATE $v\<0, t\<0, e\<0,\w_0\<\vz$
			\REPEAT
				\STATE \COMMENT{Model optimization stage}
				\STATE $\widetilde\w\<\w_e$
				\WHILE{$t < s_e$}
					\STATE Receive sample $(\x,y)$
					\STATE $\widetilde\w \< \widetilde\w + \eta_t\nabla_\w\br{(1-\frac{v_e}{2})r^+(\widetilde\w;\x,y) + \frac{v_e}{2}r^-(\widetilde\w;\x,y)}$\vskip-2ex
					\STATE $t\<t+1$
				\ENDWHILE
				\STATE $t\<0, e\<e+1, \w_{e+1}\<\widetilde\w$
				\STATE \COMMENT{Challenge level estimation stage}
				\STATE $v_+\<0,v_-\<0$
				\WHILE{$t < s'_e$}
					\STATE Receive sample $(\x,y)$
					\STATE $v_y \< v_y + r^y(\w_e;\x,y)$
					\STATE $t\<t+1$
				\ENDWHILE
				\STATE $t\<0, v_e\< \frac{2v_+}{2+v_+-v_-}$
			\UNTIL{stream is exhausted}
			\STATE \textbf{return} $\w_e$
		}
	\end{algorithmic}
\end{algorithm}

\begin{thm}
\label{thm:altmax-conv}
Let Algorithm~\ref{alg:am} be executed with a performance measure $\Pf_{(\ba,\bb)}$ and reward functions that offer values in the range $[0,m)$. Let $\Pf^\ast := \sup_{\w\in\W}\Pf_{(\ba,\bb)}(\w)$. Also let $\Delta_t = \Pf^\ast - \Pf_{(\ba,\bb)}(\w_t)$ be the excess error for the model $\w_t$ generated at time $t$. Then there exists a value $\eta(m) < 1$ such that $\Delta_t \leq \Delta_0\cdot\eta(m)^t$.
\end{thm}
The proof of this theorem can be found in Appendix~\ref{app:thm-altmax-conv-proof}. Table~\ref{tab:pseudolinear-perf-list} gives values for the convergence rates of all the pseudo-linear performance measures, as well as the allowed range of values that the reward functions can take for those measures. This is important since performance measures such as the F-measure diverge if the reward function values approach $2$. Other performance measures like the Gower-Legendre measure do not impose any such restrictions. Note that the above result shows that Algorithm~\ref{alg:am} will always terminate in $\O{\log\frac{1}{\epsilon}}$ steps.

At this point it would be apt to make a historical note. Pseudo-linear functions have enjoyed a fair amount of interest in the optimization community \cite{Schaible76,Dinkelbach67,Jagannathan66} within the sub-field of fractional programming. Of the many methods that have been developed to optimize these functions, the Dinkelbach-Jagannathan (DJ) procedure \cite{Dinkelbach67,Jagannathan66} is of specific interest to us. It turns out that the \am method can be seen as performing DJ-style updates over parameterized spaces (the parameter being the model $\w$). It is known (for instance see \cite{Schaible76}) that the DJ process is able to offer a linear convergence rates. Our proof of Theorem~\ref{thm:altmax-conv}, which was obtained independently, can then be seen as giving a similar result in the parameterized setting.

However, we wish to move one step further and optimize these performance measures in an online stochastic manner. To this end, we observe that the \am algorithm can be executed in an online fashion by using stochastic updates to train the intermediate models. The resulting algorithm \amsgd, is presented in Algorithm~\ref{alg:amsgd}. However, this algorithm is much harder to analyze because unlike \am which has the luxury of offering exact updates, \amsgd offers inexact, even noisy updates. Indeed, even existing works in the optimization community (for example \cite{Schaible76}) do not seem to have analyzed DJ-style methods with noisy updates.
\begin{figure*}[t!]
\centering\hspace*{-10pt}
\subfigure[IJCNN1]{
\includegraphics[scale=0.51]{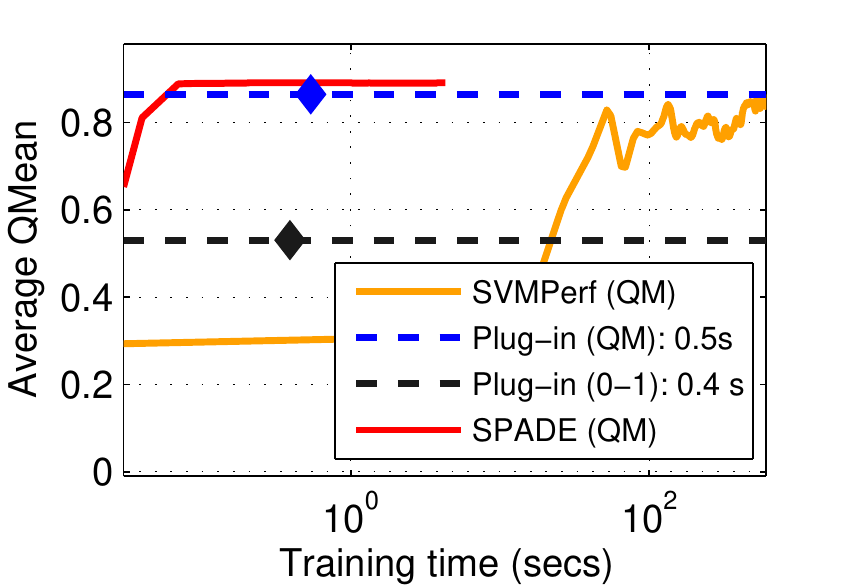}
\label{subfig:ijcnn1-QM}
}\hspace*{-10pt}
\subfigure[KDD08]{
\includegraphics[scale=0.51]{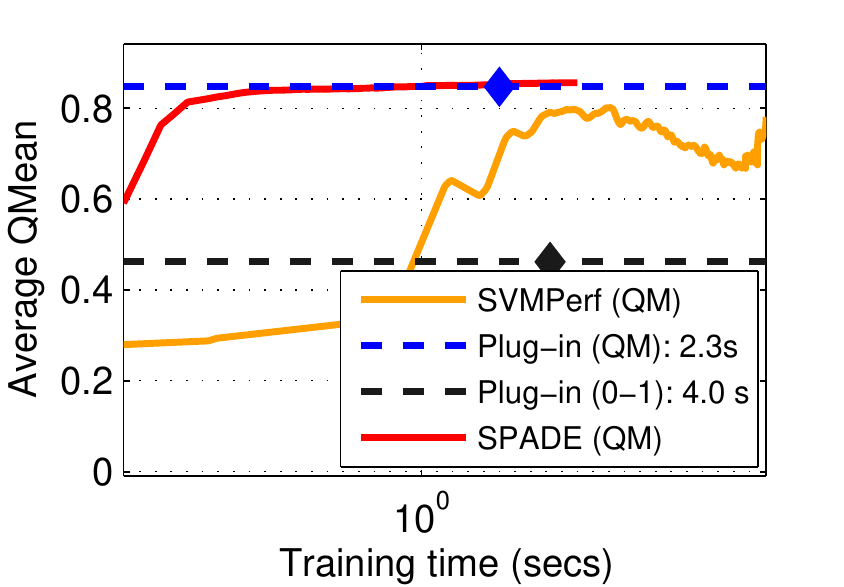}
\label{subfig:kdd08-QM}
}\hspace*{-10pt}
\subfigure[PPI]{
\includegraphics[scale=0.51]{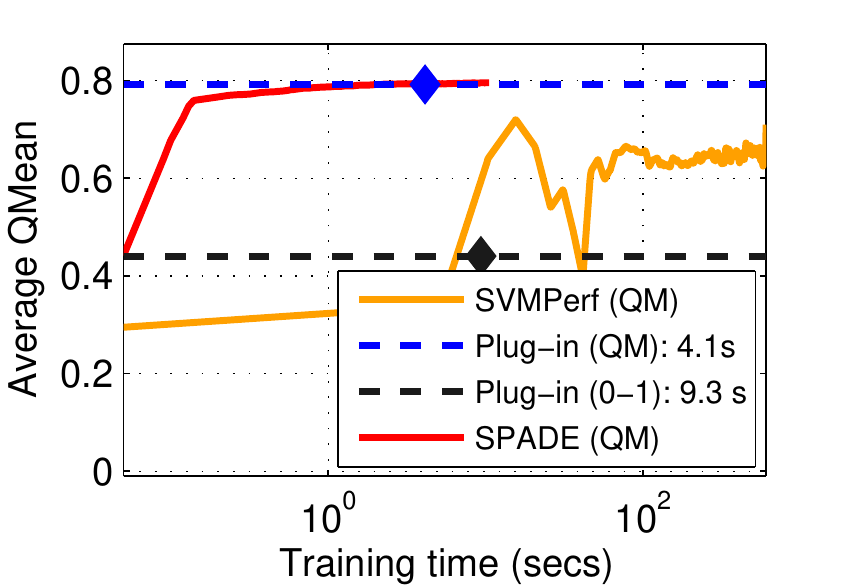}
\label{subfig:ppi-QM}
}\hspace*{-10pt}
\subfigure[Covtype]{
\includegraphics[scale=0.51]{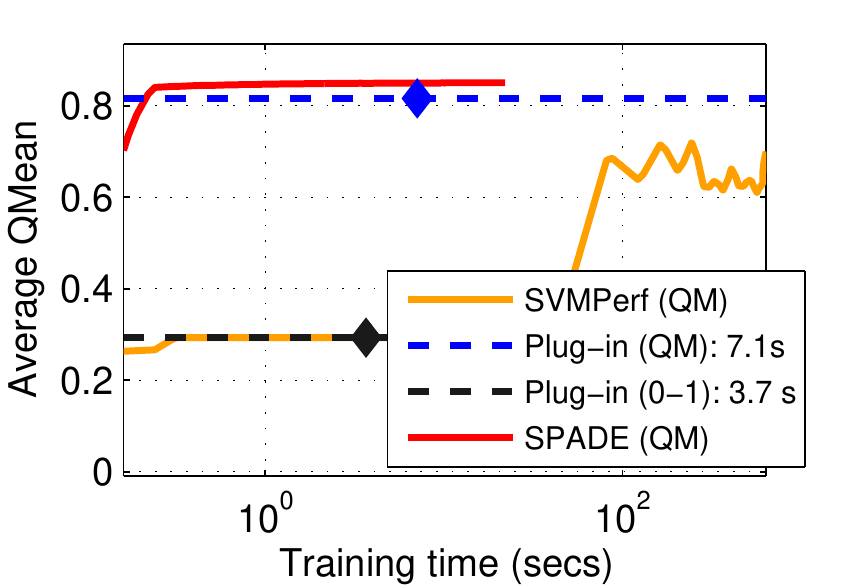}
\label{subfig:covtype-QM}
}
\caption{Comparison of stochastic primal-dual method (SPADE) with baseline methods on QMean maximization tasks. SPADE achieves similar/better accuracies while consistently requiring about 3-4x less time than other baseline approaches.}
\label{fig:QM}
\end{figure*}
\begin{figure*}[t!]
\centering\hspace*{-10pt}
\subfigure[IJCNN1]{
\includegraphics[scale=0.51]{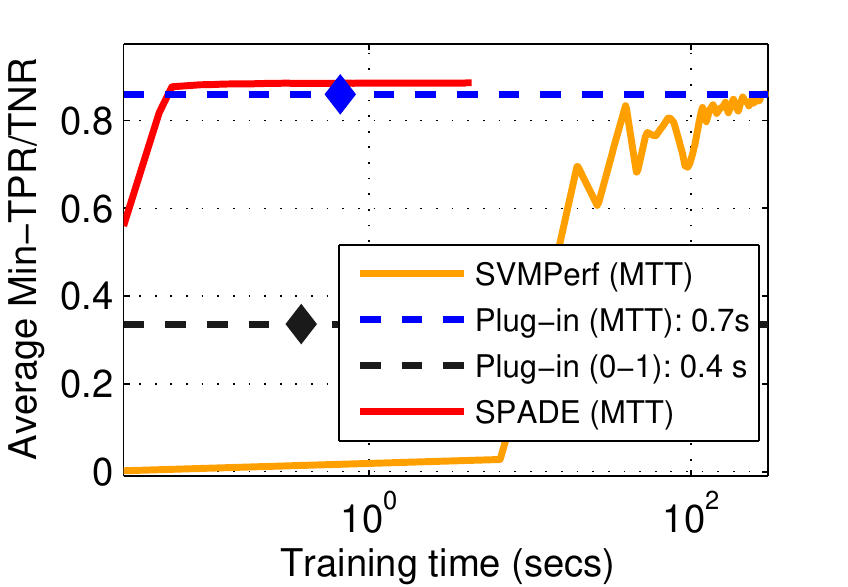}
\label{subfig:ijcnn1-MTT}
}\hspace*{-10pt}
\subfigure[KDD08]{
\includegraphics[scale=0.51]{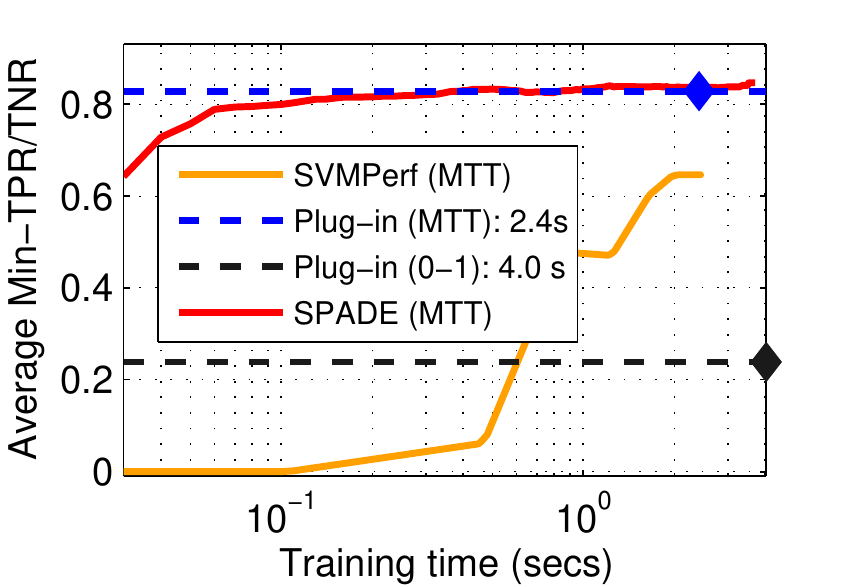}
\label{subfig:kdd08-MTT}
}\hspace*{-10pt}
\subfigure[PPI]{
\includegraphics[scale=0.51]{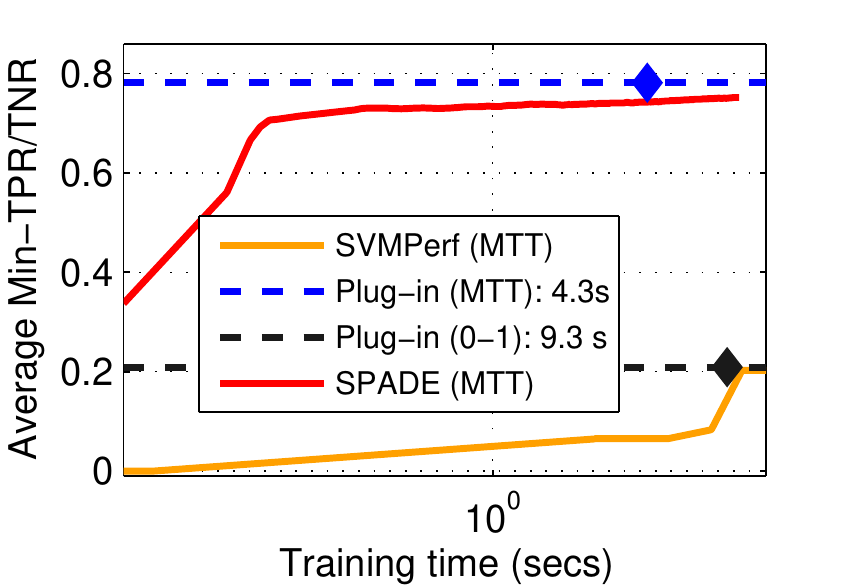}
\label{subfig:ppi-MTT}
}\hspace*{-10pt}
\subfigure[MNIST]{
\includegraphics[scale=0.51]{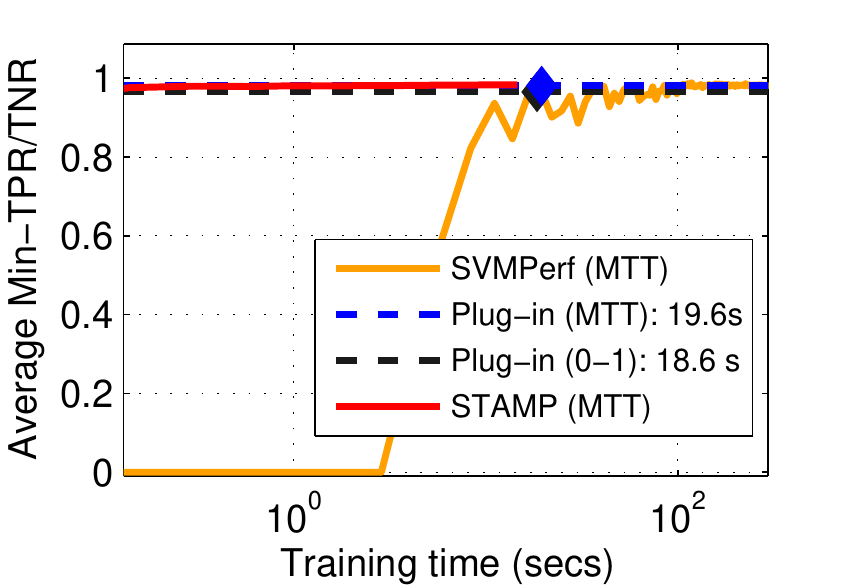}
\label{subfig:mnist-MTT}
}
\caption{Comparison of stochastic primal-dual method (SPADE) with baseline methods on Min-TPR/TNR maximization tasks}
\label{fig:MTT}
\end{figure*}

Our next contribution hence, is an analysis of the convergence rate offered by the \am algorithm when neither of the two maximizations is carried out exactly. For the sake of simplicity, we present the \amsgd algorithm and its analysis for the case of $F_1$ measure. Suppose at each time step, for some $\epsilon_t \geq 0,\delta_t$, we have
\begin{align*}
V(\w_{t+1},v_t) &= \max_{\w\in\W}\ V(\w,v_t) - \epsilon_t\\
v_t &= F(\w_t) + \delta_t,
\end{align*}
then for some $\eta < 1$, we have
\begin{align*}
\Delta_T \leq \eta^T\Delta_0 + \sum_{i=0}^{T-1}\eta^{T-i}\br{\abs{\delta_t} + \epsilon_t}
\end{align*}
 As a corollary we present a convergence analysis for the \amsgd algorithm in Theorem~\ref{thm:altmaxsgd-conv}.
\begin{thm}
\label{thm:altmaxsgd-conv}
Let Algorithm~\ref{alg:amsgd} be executed with a performance measure $\Pf_{(\ba,\bb)}$ and reward functions with range $[0,m)$. Let $\eta = \eta(m)$ be the rate of convergence guaranteed for $\Pf_{(\ba,\bb)}$ by the \am algorithm. Set the epoch lengths to $s_e,s'_e = \softO{\frac{1}{\eta^{2e}}}$. Then after $e = \log_{\frac{1}{\eta}}\br{\frac{1}{\epsilon}\log^2\frac{1}{\epsilon}}$ epochs, we can ensure with probability at least $1-\delta$ that $\Pf^\ast - \Pf_{(\ba,\bb)}(\w_e) \leq \epsilon$. Moreover the number of samples consumed till this point is at most $\softO{\frac{1}{\epsilon^2}}$.
\end{thm}
The convergence analysis for noisy \am can be found in Appendix~\ref{app:noisy-am}. The proof of this theorem can be found in Appendix~\ref{app:thm-altmaxsgd-conv-proof}. Both results require a fine grained analysis of how errors accumulate throughout the learning process.

%%% Local Variables: 
%%% mode: latex
%%% TeX-master: "sgd-tpr-tnr"
%%% End: 

%% file: expts.tex
\begin{figure*}[t]
\centering\hspace*{-10pt}
\subfigure[IJCNN1]{
\includegraphics[scale=0.51]{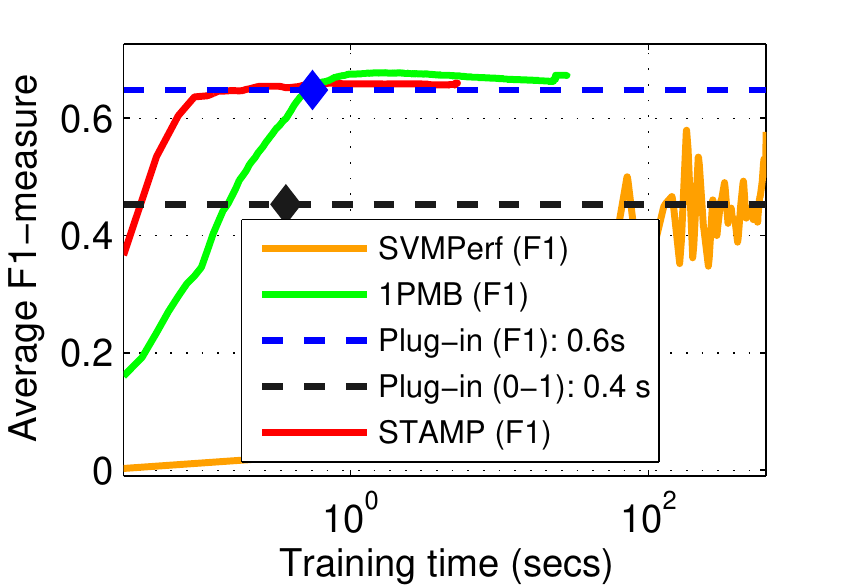}
\label{subfig:ijcnn1-F1}
}\hspace*{-10pt}
\subfigure[KDD08]{
\includegraphics[scale=0.51]{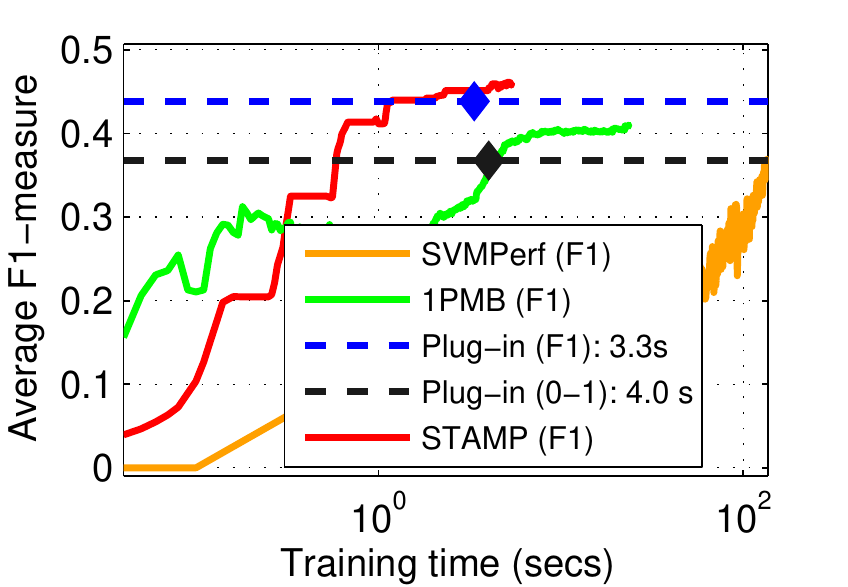}
\label{subfig:kdd08-F1}
}\hspace*{-10pt}
\subfigure[PPI]{
\includegraphics[scale=0.51]{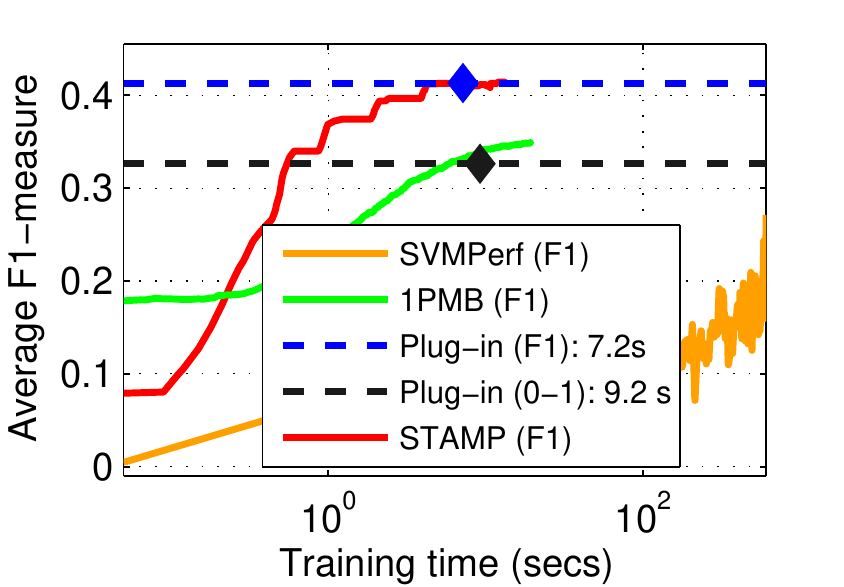}
\label{subfig:ppi-F1}
}\hspace*{-10pt}
\subfigure[Covtype]{
\includegraphics[scale=0.51]{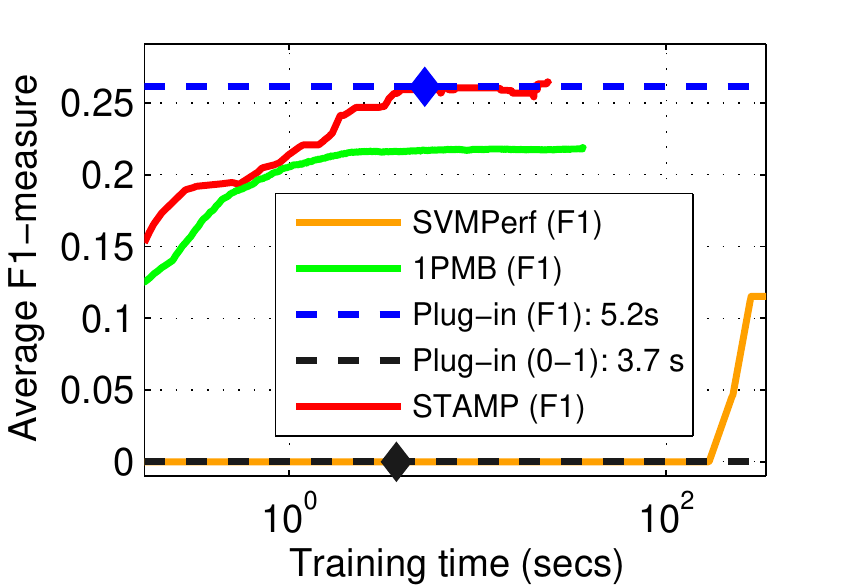}
\label{subfig:covtype-F1}
}
\caption{Comparison of stochastic alternating minimization procedure (STAMP) with baseline methods on F1 maximization tasks}%\vspace*{-12pt}
\label{fig:F1}
\end{figure*}
\begin{figure*}[t]
\centering\hspace*{-10pt}
\subfigure[IJCNN1]{
\includegraphics[scale=0.51]{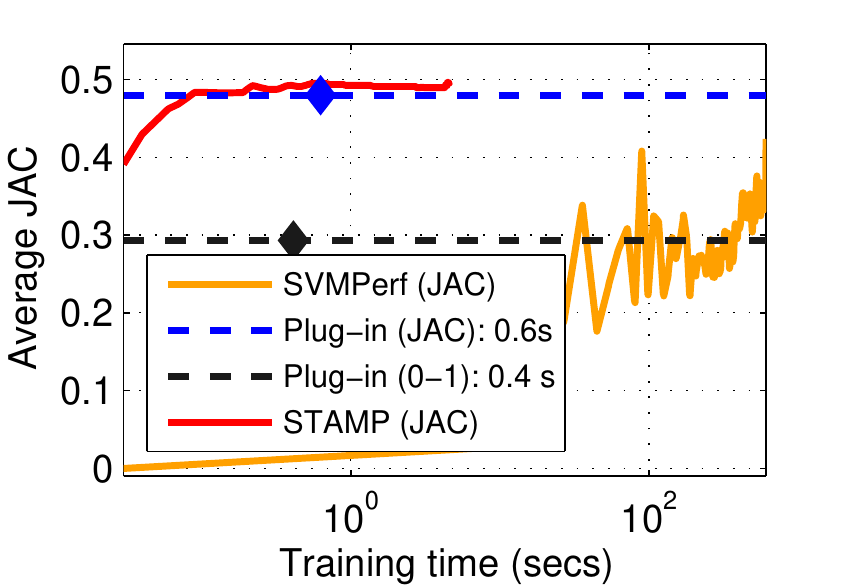}
\label{subfig:ijcnn1-JAC}
}\hspace*{-10pt}
\subfigure[KDD08]{
\includegraphics[scale=0.51]{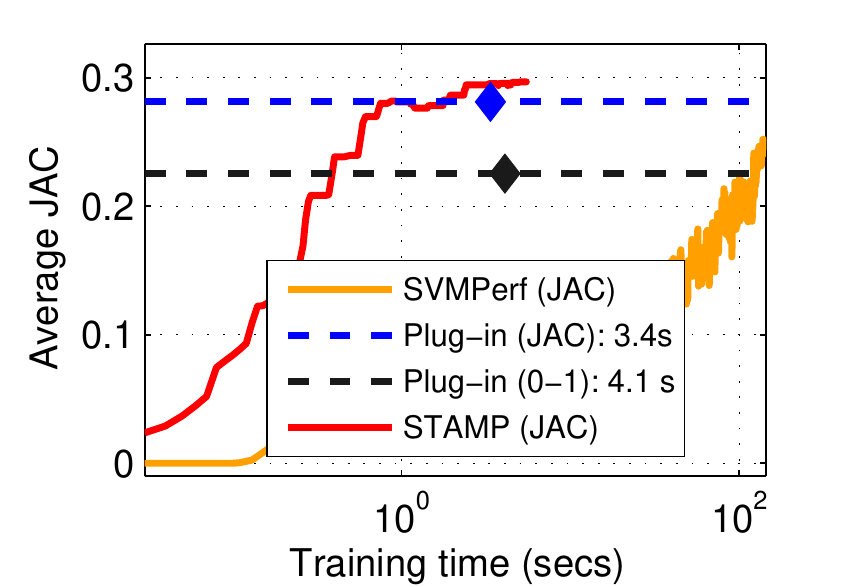}
\label{subfig:kdd08-JAC}
}\hspace*{-10pt}
\subfigure[PPI]{
\includegraphics[scale=0.51]{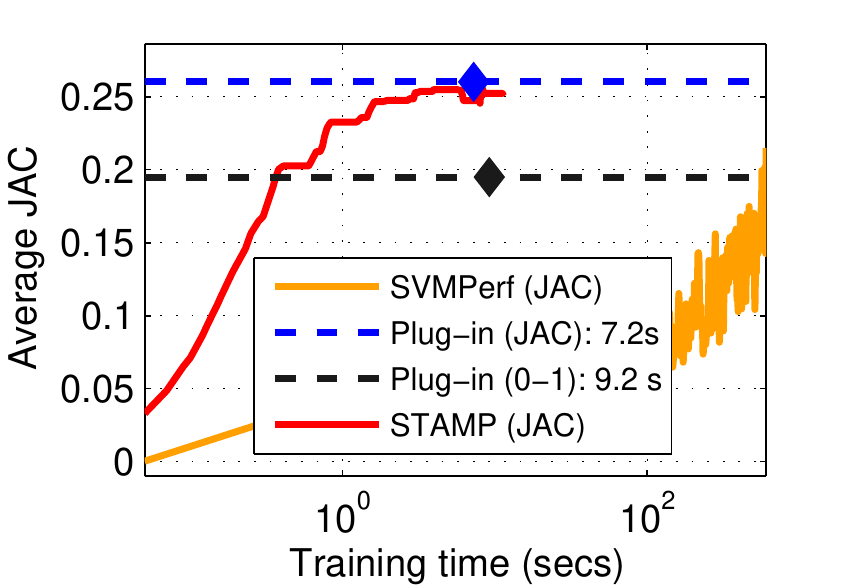}
\label{subfig:ppi-JAC}
}\hspace*{-10pt}
\subfigure[Covtype]{
\includegraphics[scale=0.51]{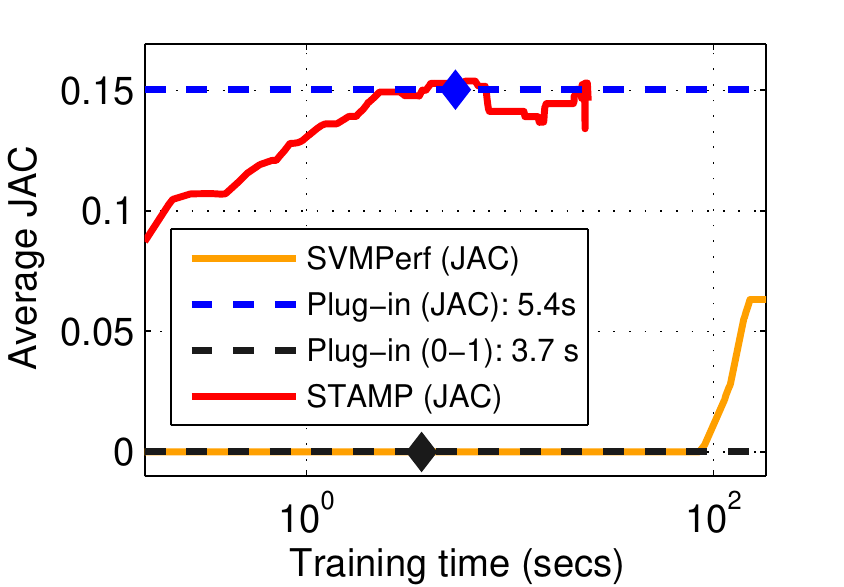}
\label{subfig:covtype-JAC}
}\vspace*{-5pt}
\caption{Comparison of stochastic alternating minimization procedure (STAMP) with baseline methods on JAC maximization tasks}
\label{fig:JAC}
\end{figure*}
\section{Experimental Results}
\label{sec:exps}
We shall now compare our methods with the state-of-the-art on various performance measures and datasets.

\textbf{Datasets}: We evaluated our methods on 5 publicly available benchmark datasets: a) PPI, b) KDD Cup 2008, c) IJCNN, d) Covertype, e) MNIST. All datasets exhibited moderate to severe label imbalance with the KDD Cup 2008 dataset having just $0.61\%$ positives.

\textbf{Methods}: We instantiated the \pdsgd algorithm (Algorithm~\ref{algo:spdu}) on the Q-mean and Min-TPR/TNR performance measures. We also instantiated the \amsgd method (Algorithm~\ref{alg:amsgd}) on F1-measure and the JAC coefficient. In both cases we compared to the SVMPerf method \cite{JoachimsFY09} and plug-in method \cite{KoyejoNRD14} specialized to these measures. For the sake of reference, we also compared to the standard logistic regression method for (unweighted) binary classification. Additionally for F1-measure, we also compared to the \spmb stochastic gradient descent method proposed recently by \cite{KarNJ14}. All methods were implemented in C.

\textbf{Parameters}:  We used $70\%$ of the dataset for training and the rest for testing. Tunable parameters, including thresholds for the plug-in approaches, were cross-validated on a validation set. All results reported here were averaged over 5 random train-test splits. We used hinge-loss based reward functions for our methods. \amsgd was executed by setting the challenge level to the actual F-measure/JAC at each stage. We used a state of the art LBFGS solver to implement the plug-in methods and used standard implementations of the SVMPerf algorithm. Since our methods are able to take a single pass over the data very rapidly, \pdsgd was allowed to run for 25 passes over the data and \amsgd was allowed 25 passes with an initial epoch length of $100$ which was doubled after every iteration. The SVMPerf algorithm was allowed a runtime of up to $50\times$ of that given to our method after which it was terminated. The LBFGS solver was always allowed to run till convergence. 

Figures~\ref{fig:QM}~and~\ref{fig:MTT} compare the \pdsgd method with the baseline methods for the Q-mean and Min-TPR/TNR measures. In general, \pdsgd was found to offer comparable or superior accuracies with greatly accelerated convergence as compared to other methods. On the IJCNN and Covtype datasets, \pdsgd outperformed every other method by about 2-3\%. As \pdsgd is a stochastic first order method, it is expected to rapidly find out a fairly accurate solution.  Indeed, the method was found to offer greatly accelerated convergence without fail. For instance, on the MNIST dataset, \pdsgd found out the best solution as much as $60\times$ faster than any other method whereas on the KDD Cup and PPI datasets it was $12\times$ and $2\times$ faster respectively. The SVMPerf method, on the other hand, was found to be extremely slow in general and require at least an order of magnitude time more than \pdsgd to find reasonably accurate solutions. It is also notable that in all cases, simple binary classification gave very poor accuracies due to the severe label imbalance in these datasets.

Figures~\ref{fig:F1}~and~\ref{fig:JAC} report the performance of the \amsgd method applied to pseudo-linear functions. Similar to the concave measures, \amsgd was found to provide competitive accuracies as compared to the baseline methods but require at least $3-4\times$ less computational time. Interestingly, for the F1-measure, the \spmb method, which is another stochastic gradient descent-based method, was found to struggle to obtain accuracies similar to that of \amsgd or else offer much slower convergence. We suspect two main reasons for the suboptimal behavior of this other stochastic method. Firstly these results confirm the adverse effect of the dependence on an in-memory buffer on these methods. It is notable that this dependence causes even the theoretical convergence rates for these methods to be weaker as was noted earlier in the discussion. Secondly, we note that both SVMPerf and \spmb optimize the same ``struct-SVM'' style surrogate for the F-measure \cite{KarNJ14}. This surrogate has been observed to give poor accuracies when compared to plug-in methods in several previous works \cite{KoyejoNRD14,NarasimhanVA14}. \amsgd on the other hand, works directly with F-measure in a manner similar to, but faster than, the plug-in methods which might explain its better performance.

%%% Local Variables: 
%%% mode: latex
%%% TeX-master: "sgd-tpr-tnr"
%%% End: 

%% file: app.tex
\section{Proof of Lemma~\ref{lem:dsr-stab}}
\label{app:lem-dsr-stab-proof}
\begin{replem}{lem:dsr-stab}
The stability parameter of a performance measure $\Psi(\cdot)$ can be written as $\delta(\epsilon) \leq L_\Psi\cdot\epsilon$ iff its sufficient dual region is bounded in a ball of radius $\Theta\br{L_\Psi}$.
\end{replem}
\begin{proof}
Let us denote primal variables using the notation $\x = (u,v)$ and dual variables using the notation $\btheta = (\alpha,\beta)$. The proof follows from the fact that any value of $\btheta$ for which $\Psi^\ast(\btheta) = -\infty$ can be safely excluded from the sufficient dual region.

For proving the result in one direction suppose $\Psi$ is stable with $\delta(\epsilon) = L\epsilon$ for some $L > 0$. Now consider some $\btheta \in \R^2$ such that $\norm{\btheta}_2 \geq L$. Now set $\x_C = -C\cdot\btheta$. Then we have
\begin{align*}
\Psi^\ast(\btheta) &= \inf_{\x}\bc{\ip{\btheta}{\x} - \Psi(\x)}\\
				   &\leq \inf_{C > 0}\bc{\ip{\btheta}{\x_C} - \Psi(\x_C)}\\
				   &= \inf_{C > 0}\bc{-C\norm{\btheta}_2^2 - \Psi(\x_C)}\\
				   &\leq \inf_{C > 0}\bc{-C\norm{\btheta}_2^2 - \Psi(\vz) + CL\norm{\btheta}_\infty}\\
				   &\leq \inf_{C > 0}\bc{-C\norm{\btheta}_2^2 - \Psi(\vz) + CL\norm{\btheta}_2}\\
				   &= \inf_{C > 0}\bc{-C\norm{\btheta}_2(\norm{\btheta}_2 - L)} - \Psi(\vz)\\
				   &\leq \inf_{C > 0}\bc{-C\norm{\btheta}_2 - \Psi(\vz)}\\
				   &= -\infty
\end{align*}
Thus, we can conclude that no dual vector with norm greater than $L$ can be a part of the sufficient dual region. This shows that the sufficient dual region is bounded inside a ball of radius $L$. For proving the result in the other direction, suppose the dual sufficient region is indeed bounded in a ball of radius $R$. Consider two points $\x_1,\x_2$ such that
\begin{align*}
\btheta^\ast_1 &= \mathop{\arg\min}_{\btheta\in\A_\Psi}\bc{\ip{\btheta}{\x_1} - \Psi^\ast(\btheta)}\\
\btheta^\ast_2 &= \mathop{\arg\min}_{\btheta\in\A_\Psi}\bc{\ip{\btheta}{\x_2} - \Psi^\ast(\btheta)}\\
\end{align*}
Now define $f(\btheta,\x) := \ip{\btheta}{\x} - \Psi^\ast(\btheta)$ so that, by the above definition, $f(\btheta^\ast_1,\x_1) = \Psi(\x_1)$ and $f(\btheta^\ast_2,\x_2) = \Psi(\x_2)$. Now we have
\begin{align*}
\Psi(\x_1) &= f(\btheta^\ast_1,\x_1) \leq f(\btheta^\ast_2,\x_1)\\
		   &\leq f(\btheta^\ast_2,\x_2) + \abs{\ip{\btheta^\ast_2}{\x_1-\x_2}}\\
		   &= \Psi(\x_2) + \abs{\ip{\btheta^\ast_2}{\x_1-\x_2}}\\
		   &\leq \Psi(\x_2) + R\norm{\x_1-\x_2}_2,
\end{align*}
where the fourth step follows from the norm bound on $\btheta^\ast_2$. Similarly we have
\[
\Psi(\x_2) \leq \Psi(\x_1) + R\norm{\x_1-\x_2}_2
\]
This establishes the result.
\end{proof}

\section{Proof of Theorem~\ref{thm:pdsgd-risk-analysis}}
\label{app:thm-pdsgd-risk-analysis-proof}
\begin{repthm}{thm:pdsgd-risk-analysis}
Suppose we are given a stream of random samples $(\x_1,y_1),\ldots,(\x_T,y_T)$ drawn from a distribution $\D$ over $\X\times\Y$. Let $\Psi(\cdot)$ be a concave, Lipschitz link function. Let Algorithm~\ref{algo:spdu} be executed with a dual feasible set $\A \supseteq \A_\Psi$, $\eta_t = 1/\sqrt t$ and $\eta'_t = 1/\sqrt t$. Then, the average model $\barw = \frac{1}{T}\sum_{t=1}^T \w_t$ output by the algorithm satisfies, with probability at least $1 - \delta$,
\[
\Pf_\Psi(\barw) \geq \sup_{\w^\ast \in \W}\Pf_\Psi(\w^\ast) -  \delta_\Psi\br{\sqrt{\frac{2B_r^2}{T}\log\frac{1}{\delta}}} - \br{L_\Psi^2 + 4B_r^2}\frac{1}{2\sqrt T} - \br{L_\Psi^2L_r^2 + R_\W^2}\frac{1}{2\sqrt T} - \sqrt{\frac{2L_\Psi^2B_r^2}{T}\log\frac{1}{\delta}}.
\]
\end{repthm}
\begin{proof}
For this proof we shall assume that $\Psi$ is $L_\Psi$-Lipschitz so that its sufficient dual region can be bounded by an application of Lemma~\ref{lem:dsr-stab}. Notice that the updates for $(\alpha,\beta)$ can be written as follows:
\[
(\alpha_{t+1},\beta_{t+1}) \leftarrow \Pi_{\A_\Psi}\br{(\alpha_{t},\beta_{t}) - \eta_t\nabla_{(\alpha,\beta)}\ell^d_t(\alpha_t,\beta_t)},
\]
where
\[
\ell^d_t(\alpha,\beta) = \left\{
  \begin{array}{l l}
    \alpha r^+(\w_t; \x_t, y_t) - \Psi^*(\alpha, \beta) & \quad \text{if $y_t > 0$}\\
    \beta r^-(\w_t; \x_t, y_t) - \Psi^*(\alpha, \beta) & \quad \text{if $y_t < 0$}
  \end{array} \right.
\]
which can be interpreted as simple gradient descent with $\ell_t$. Moreover, since $\Psi^\ast$ is concave, $\ell^d_t$ is convex with respect to $(\alpha, \beta)$ for every $t$. Note that the terms $r^+(\w_t; \x_t, y_t)$ and $r^-(\w_t; \x_t, y_t)$ do not involve $\alpha,\beta$ and hence act as arbitrary bounded positive constants for this part of the analysis.

Note that by Lemma~\ref{lem:dsr-stab}, we have the radius of $\A_\Psi$ bounded by $L_\Psi$. Also, since $\Psi$ is a monotone function, by a similar argument, $\Psi^*(\alpha, \beta)$ can be shown to be a $\Psi(B_r,B_r)$-Lipschitz function. For all the performance measures considered, we have $\Psi(B_r,B_r) \leq B_r$. Thus, $\ell^d_t(\alpha,\beta)$ is a $2B_r$-Lipschitz function. Hence, using a standard GIGA-style analysis \cite{zinkevich} on the (descent) updates on $\alpha_t$ and $\beta_t$ in Algorithm~\ref{algo:spdu}, we have (for $\eta_t =  \frac{1}{\sqrt t}$)
\begin{eqnarray}
\lefteqn{
\frac{1}{T}\sum_{t=1}^T \big[\alpha_t r^+(\w_t; \,\x_t, y_t) \,+\, \beta_t r^-(\w_t; \,\x_t, y_t) \,-\, \Psi^*(\alpha_t, \beta_t)\big]
}\hspace{1cm}
\nonumber
\\
&\leq& \inf_{(\alpha, \beta) \in \A} \Big\{
	\frac{1}{T}\sum_{t=1}^T \big[\alpha r^+(\w_t; \,\x_t, y_t) \,+\, \beta r^-(\w_t; \,\x_t, y_t) \,-\, \Psi^*(\alpha, \beta)\big]
			\Big\} ~+~ \br{L_\Psi^2 + 4B_r^2}\frac{1}{2\sqrt T} \nonumber\\
&=& \inf_{(\alpha, \beta) \in \A} \bigg\{
	\alpha \frac{1}{T}\sum_{t=1}^T r^+(\w_t; \,\x_t, y_t) \,+\, \beta \frac{1}{T}\sum_{t=1}^T  r^-(\w_t; \,\x_t, y_t) \,-\, \Psi^*(\alpha, \beta)
			\bigg\} ~+~ \br{L_\Psi^2 + 4B_r^2}\frac{1}{2\sqrt T} \nonumber\\
&=& \Psi\bigg(\frac{1}{T}\sum_{t=1}^T r^+(\w_t; \,\x_t, y_t), \,\frac{1}{T}\sum_{t=1}^T  r^-(\w_t; \,\x_t, y_t)\bigg) ~+~ \br{L_\Psi^2 + 4B_r^2}\frac{1}{2\sqrt T},
\nonumber
\end{eqnarray}
where the last step follows from Fenchel conjugacy.

Further, noting that $\Eemp{\x_{t}, y_{t}}{r^+(\w_t; \,\x_t, y_t) \,\big|\, \x_{1:t-1}, y_{1:t-1}}  \,=\, P(\w_t)$, and $\Eemp{\x_{t}, y_{t}}{r^-(\w_t; \,\x_t, y_t) \,\big|\, \x_{1:t-1}, y_{1:t-1}}$ $\,=\, N(\w_t)$, we use the standard online-batch conversion bounds \cite{online-batch-single} to the loss functions $r^+$ and $r^-$ individually to obtain w.h.p.
\[
\frac{1}{T}\sum_{t=1}^T r^+(\w_t; \,\x_t, y_t) \leq \sum_{t=1}^T P(\w_t) + \sqrt{\frac{2B_r^2}{T}\log\frac{1}{\delta}}
\]
\[
\frac{1}{T}\sum_{t=1}^T  r^-(\w_t; \,\x_t, y_t) \leq \sum_{t=1}^T N(\w_t) + \sqrt{\frac{2B_r^2}{T}\log\frac{1}{\delta}}
\]
By monotonicity of $\Psi$, we get
\begin{eqnarray}
\lefteqn{
\frac{1}{T}\sum_{t=1}^T \big[\alpha_t r^+(\w_t; \,\x_t, y_t) \,+\, \beta_t r^-(\w_t; \,\x_t, y_t) \,-\, \Psi^*(\alpha_t, \beta_t)\big]
}\hspace{1cm}
\nonumber
\\
&\leq&
\Psi\bigg(\frac{1}{T}\sum_{t=1}^T P(\w_t) +  \sqrt{\frac{2B_r^2}{T}\log\frac{1}{\delta}}, ~ \frac{1}{T}\sum_{t=1}^T  N(\w_t) +  \sqrt{\frac{2B_r^2}{T}\log\frac{1}{\delta}} \bigg)  \,+\,  \br{L_\Psi^2 + 4B_r^2}\frac{1}{2\sqrt T}
\nonumber
\\
&\leq&
\Psi\bigg(\frac{1}{T}\sum_{t=1}^T P(\w_t), ~ \frac{1}{T}\sum_{t=1}^T  N(\w_t) \bigg)  \,+\,  \delta_\Psi\bigg(\sqrt{\frac{2B_r^2}{T}\log\frac{1}{\delta}}\bigg) \,+\, \br{L_\Psi^2 + 4B_r^2}\frac{1}{2\sqrt T}
\nonumber
\\
&\leq&
\Psi\bigg(\bar{r}^+\bigg(\frac{1}{T}\sum_{t=1}^T \w_t\bigg), ~ \bar{r}^-\bigg(\frac{1}{T}\sum_{t=1}^T \w_t\bigg) \bigg)  \,+\,  \delta_\Psi\bigg(\sqrt{\frac{2B_r^2}{T}\log\frac{1}{\delta}}\bigg) \,+\, \br{L_\Psi^2 + 4B_r^2}\frac{1}{2\sqrt T}
\nonumber
\\
&=&
\Psi\big(P(\barw), ~  N(\barw) \big)  \,+\,  \delta_\Psi\bigg(\sqrt{\frac{2B_r^2}{T}\log\frac{1}{\delta}}\bigg) \,+\, \br{L_\Psi^2 + 4B_r^2}\frac{1}{2\sqrt T},
\label{eqn:giga-1}
\end{eqnarray}
where the second inequality follows from stability of $\Psi$, and the third inequality follows from concavity of $\bar{r}^+$ and $\bar{r}^-$, Jensen's inequality, and stability of $\Psi$.

Similarly, the update to $\w$ can be written as
\[
\w_{t+1} \leftarrow \Pi_\W\br{\w_t - \eta'_t \nabla_\w\ell^p_t(\w_t)}, 
\]
where $\Pi_\W$ is the projection operator for the domain $\W$ and
\[
\ell^p_t(\w) = \left\{
  \begin{array}{l l}
    -\alpha_t r^+(\w; \x_t, y_t)  + \Psi^*(\alpha_t, \beta_t) & \quad \text{if $y_t > 0$}\\
    -\beta_t r^-(\w; \x_t, y_t) + \Psi^*(\alpha_t, \beta_t) & \quad \text{if $y_t < 0$}
  \end{array} \right.
\]
Since $r^+,r^-$ are concave and the term $\Psi^*(\alpha_t, \beta_t)$ does not involve $\w$, $\ell^p_t$ is convex in $\w$ for all $t$. Also, we can show that $\ell^p_t(\w)$ is an $\br{L_\Psi\cdot L_r}$-Lipschitz function. Hence, applying a standard GIGA analysis \cite{zinkevich} to the (ascent) update on $\w_t$ in Algorithm 1 (with $\eta_t' = \frac{1}{\sqrt t}$), we have for any $\w^* \in \W$,
\begin{eqnarray*}
\lefteqn{\frac{1}{T}\sum_{t=1}^T \big[\alpha_t r^+(\w_t; \x_t, y_t) \,+\, \beta_t r^-(\w_t; \x_t, y_t) \,-\, \Psi^*(\alpha_t, \beta_t)\big]}
\nonumber\\
	& \geq & \frac{1}{T}\sum_{t=1}^T \big[\alpha_t r^+(\w^*; \x_t, y_t)
				\,+\, \beta_t r^-(\w^*; \x_t, y_t) \,-\, \Psi^*(\alpha_t, \beta_t)\big] ~-~ \br{L_\Psi^2L_r^2 + R_\W^2}\frac{1}{2\sqrt T}. \nonumber
\end{eqnarray*}
Again, observing that by linearity of expectation, we have
\[
\Eemp{\x_{t}, y_{t}}{\alpha_t r^+(\w^*; \x_t, y_t) + \beta_t r^-(\w^*; \x_t, y_t) \,\big|\, \x_{1:t-1}, y_{1:t-1}}  \,=\, \alpha_t P(\w^*) + \beta_t N(\w^*),
\]
which gives us, through an online-batch conversion argument \cite{online-batch-single} w.h.p,
\begin{eqnarray}
\lefteqn{\frac{1}{T}\sum_{t=1}^T \big[\alpha_t r^+(\w_t; \x_t, y_t) \,+\, \beta_t r^-_t(\w_t; \x_t, y_t) \,-\, \Psi^*(\alpha_t, \beta_t)\big]} 
\hspace{1cm}
\nonumber\\
	& \geq & \frac{1}{T}\sum_{t=1}^T \big[\alpha_t P(\w^*)
				\,+\, \beta_t N(\w^*)\big] \,-\, \frac{1}{T}\sum_{t=1}^T  \Psi^*(\alpha_t, \beta_t)  ~-~ \sqrt{\frac{2L_\Psi^2B_r^2}{T}\log\frac{1}{\delta}} - \br{L_\Psi^2L_r^2 + R_\W^2}\frac{1}{2\sqrt T}\nonumber\\
	& \geq & \frac{1}{T}\sum_{t=1}^T \big[\alpha_t P(\w^*)
				\,+\, \beta_t N(\w^*)\big] \,-\, \Psi^*\bigg(\frac{1}{T}\sum_{t=1}^T \alpha_t,  \frac{1}{T}\sum_{t=1}^T \beta_t \bigg)  ~-~  \sqrt{\frac{2L_\Psi^2B_r^2}{T}\log\frac{1}{\delta}} - \br{L_\Psi^2L_r^2 + R_\W^2}\frac{1}{2\sqrt T}\nonumber\\
	& = & \bar{\alpha} P(\w^*)
				\,+\, \bar{\beta}N(\w^*) \,-\, \Psi^*(\bar{\alpha},  \bar{\beta})  ~-~ \sqrt{\frac{2L_\Psi^2B_r^2}{T}\log\frac{1}{\delta}} - \br{L_\Psi^2L_r^2 + R_\W^2}\frac{1}{2\sqrt T}\nonumber\\
	&\geq& \inf_{\alpha, \beta} \Big\{ \alpha P(\w^*)
				\,+\, \beta N(\w^*) \,-\, \Psi^*(\alpha,  \beta)\Big\}  ~-~ \sqrt{\frac{2L_\Psi^2B_r^2}{T}\log\frac{1}{\delta}} - \br{L_\Psi^2L_r^2 + R_\W^2}\frac{1}{2\sqrt T}\nonumber\\
	&=& \Psi\big( P(\w^*), \,N(\w^*)\big)  ~-~ \sqrt{\frac{2L_\Psi^2B_r^2}{T}\log\frac{1}{\delta}} - \br{L_\Psi^2L_r^2 + R_\W^2}\frac{1}{2\sqrt T},\label{eqn:giga-2}
\end{eqnarray}
where the second step follows from concavity of $\Psi$ and Jensen's inequality, in the third step $\bar{\alpha} = \frac{1}{T}\sum_{t=1}^T \alpha_t$ and $\bar{\beta} = \frac{1}{T}\sum_{t=1}^T \beta_t$, and the last step follows from Fenchel conjugacy.

Combining Eq. (\ref{eqn:giga-1}) and (\ref{eqn:giga-2}) gives us the desired result.
\end{proof}

\section{Proof of Theorem~\ref{thm:pdsgd-risk-analysis-non-lip}}
\label{thm-pdsgd-risk-analysis-non-lip-proof}
\begin{repthm}{thm:pdsgd-risk-analysis-non-lip}
Suppose we have the problem setting in Theorem~\ref{thm:pdsgd-risk-analysis} with the $\Psi_{\text{G-mean}}$ performance measure being optimized for. Consider a modification to Algorithm~\ref{algo:spdu} wherein the reward functions are changed to $r^+_t(\cdot) = r^+(\cdot) + \epsilon(t)$, and $r^-_t(\cdot) = r^-(\cdot) + \epsilon(t)$ for $\epsilon(t) = \frac{1}{t^{1/4}}$. Then, the average model $\barw = \frac{1}{T}\sum_{t=1}^T \w_t$ output by the algorithm satisfies, with probability at least $1 - \delta$,
\[
\Pf_{\Psi_{\text{G-mean}}}(\barw) \geq \sup_{\w^\ast \in \W}\Pf_{\Psi_{\text{G-mean}}}(\w^\ast) - \softO{\frac{1}{T^{1/4}}}.
\]
\end{repthm}
\begin{proof}
Suppose $\Psi(u+\epsilon,v+\epsilon) \leq \Psi(u,v) + \delta_\Psi(\epsilon)$ as before. Let $r^+_t(\cdot) = r^+(\cdot) + \epsilon(t)$, and $r^-_t(\cdot) = r^-(\cdot) + \epsilon(t)$. Let us make all updates with respect to $r^+_t,r^-_t$. Let $r(\epsilon)$ be the radius of the sufficient dual domain $\A$ for a given regularization $\epsilon$. Also let $\bar\epsilon = \frac{1}{T}\sum_{i=1}^T\epsilon(t)$. We will assume throughout that $\epsilon(t) = O(1)$. Then we have:

\begin{eqnarray}
\lefteqn{
\frac{1}{T}\sum_{t=1}^T \big[\alpha_t r^+_t(\w_t; \,\x_t, y_t) \,+\, \beta_t r^-_t(\w_t; \,\x_t, y_t) \,-\, \Psi^*(\alpha_t, \beta_t)\big]
}\hspace{1cm}
\nonumber
\\
&\leq& \inf_{(\alpha, \beta) \in \A} \Big\{
	\frac{1}{T}\sum_{t=1}^T \big[\alpha r^+_t(\w_t; \,\x_t, y_t) \,+\, \beta r^-_t(\w_t; \,\x_t, y_t) \,-\, \Psi^*(\alpha, \beta)\big]
			\Big\} ~+~ \O{\frac{r(\bar\epsilon)}{\sqrt{T}}} \nonumber\\
&=& \inf_{(\alpha, \beta) \in \A} \bigg\{
	\alpha \frac{1}{T}\sum_{t=1}^T r^+(\w_t; \,\x_t, y_t) + \bar\epsilon \,+\, \beta \frac{1}{T}\sum_{t=1}^T  r^-(\w_t; \,\x_t, y_t) + \bar\epsilon \,-\, \Psi^*(\alpha, \beta)
			\bigg\} ~+~ \O{\frac{r(\bar\epsilon)}{\sqrt{T}}} \nonumber\\
&=& \Psi\bigg(\frac{1}{T}\sum_{t=1}^T r^+(\w_t; \,\x_t, y_t) + \bar\epsilon, \,\frac{1}{T}\sum_{t=1}^T  r^-(\w_t; \,\x_t, y_t) + \bar\epsilon \bigg) ~+~ \O{\frac{r(\bar\epsilon)}{\sqrt{T}}}\\
&=& \Psi\bigg(\frac{1}{T}\sum_{t=1}^T r^+(\w_t; \,\x_t, y_t), \,\frac{1}{T}\sum_{t=1}^T  r^-(\w_t; \,\x_t, y_t) \bigg) ~+~ \delta_\Psi(\bar\epsilon) ~+~ \O{\frac{r(\bar\epsilon)}{\sqrt{T}}}
\nonumber
\end{eqnarray}
We can now use online to batch conversion bounds \cite{online-batch-single}, and monotonicity of $\Psi$ to get
\begin{eqnarray}
\lefteqn{
\frac{1}{T}\sum_{t=1}^T \big[\alpha_t r^+(\w_t; \,\x_t, y_t) \,+\, \beta_t r^-(\w_t; \,\x_t, y_t) \,-\, \Psi^*(\alpha_t, \beta_t)\big]
}\hspace{1cm}
\nonumber
\\
&\leq&
\Psi\big(P(\barw), ~  N(\barw) \big)  \,+\,  \delta_\Psi\bigg(\softO{\frac{1}{\sqrt{T}}}\bigg) \,+\, \delta_\Psi(\bar\epsilon) ~+~ \O{\frac{r(\bar\epsilon)}{\sqrt{T}}},
\label{eqn:giga-11}
\end{eqnarray}

For the primal updates, we get, for any $\w^* \in \W$,
\begin{eqnarray*}
\lefteqn{\frac{1}{T}\sum_{t=1}^T \big[\alpha_t r^+_t(\w_t; \x_t, y_t) \,+\, \beta_t r^-_t(\w_t; \x_t, y_t) \,-\, \Psi^*(\alpha_t, \beta_t)\big]}
\nonumber\\
	& \geq & \frac{1}{T}\sum_{t=1}^T \big[\alpha_t r^+_t(\w^*; \x_t, y_t)
				\,+\, \beta_t r^-_t(\w^*; \x_t, y_t) \,-\, \Psi^*(\alpha_t, \beta_t)\big] ~-~ \softO{\frac{r(\bar\epsilon)}{\sqrt{T}}} \\
				& = & \frac{1}{T}\sum_{t=1}^T \big[\alpha_t r^+(\w^*; \x_t, y_t)
				\,+\, \beta_t r^-(\w^*; \x_t, y_t) \,-\, \Psi^*(\alpha_t, \beta_t)\big] + \frac{1}{T}\sum_{t=1}^T\epsilon(t)(\alpha_t+\beta_t) ~-~ \softO{\frac{r(\bar\epsilon)}{\sqrt{T}}}\\
				& \geq & \frac{1}{T}\sum_{t=1}^T \big[\alpha_t r^+(\w^*; \x_t, y_t)
				\,+\, \beta_t r^-(\w^*; \x_t, y_t) \,-\, \Psi^*(\alpha_t, \beta_t)\big] ~-~ \softO{\frac{r(\bar\epsilon)}{\sqrt{T}}},
\end{eqnarray*}
since $\epsilon(t), \alpha_t, \beta_t \geq 0$. Again using an online-batch conversion argument \cite{online-batch-single} we get w.h.p,
\begin{eqnarray}
\frac{1}{T}\sum_{t=1}^T \big[\alpha_t r^+(\w_t; \x_t, y_t) \,+\, \beta_t r^-_t(\w_t; \x_t, y_t) \,-\, \Psi^*(\alpha_t, \beta_t)\big] \geq \Psi\big( P(\w^*), \,N(\w^*)\big)  ~-~ \softO{\frac{r(\bar\epsilon)}{\sqrt{T}}}.\label{eqn:giga-22}
\end{eqnarray}
Combining Eq. (\ref{eqn:giga-11}) and (\ref{eqn:giga-22}) gives us
\[
\Psi\big(P(\barw), ~  N(\barw) \big) \geq \Psi\big( P(\w^*), \,N(\w^*)\big)  ~-~ \softO{\frac{r(\bar\epsilon)}{\sqrt{T}}} - \delta_\Psi\bigg(\softO{\frac{1}{\sqrt{T}}}\bigg) \,-\, \delta_\Psi(\bar\epsilon)
\]

For G-mean, $\delta_\Psi(x) = \sqrt x$, and by an application of Lemma~\ref{lem:dsr-stab},we have $r(\epsilon) = O(1/\sqrt\epsilon)$. Thus we have
\[
\Psi\big(P(\barw), ~  N(\barw) \big) \geq \Psi\big( P(\w^*), \,N(\w^*)\big)  ~-~ \softO{\frac{1}{\sqrt{T\bar\epsilon}}} - \softO{\frac{1}{\sqrt[4]{T}}} \,-\, \sqrt{\bar\epsilon}
\]
For $\bar\epsilon = \O{\frac{1}{\sqrt[4]T}}$, we get
\[
\Psi\big(P(\barw), ~  N(\barw) \big) \geq \Psi\big( P(\w^*), \,N(\w^*)\big)  ~-~ \softO{\frac{1}{\sqrt[4]{T}}}
\]
This can be achieved with $\epsilon(t) = \frac{1}{\sqrt[4]t}$.
\end{proof}

\section{Proof of Theorem~\ref{thm:altmax-conv}}
\label{app:thm-altmax-conv-proof}
\begin{repthm}{thm:altmax-conv}
Let Algorithm~\ref{alg:am} be executed with a performance measure $\Pf_{(\ba,\bb)}$ and reward functions that offer values in the range $[0,m)$. Let $\Pf^\ast := \sup_{\w\in\W}\Pf_{(\ba,\bb)}(\w)$. Also let $\Delta_t = \Pf^\ast - \Pf_{(\ba,\bb)}(\w_t)$ be the excess error for the model $\w_t$ generated at time $t$. Then there exists a value $\eta(m) < 1$ such that for $\Delta_t \leq \Delta_0\cdot\eta(m)^t$.
\end{repthm}
\begin{proof}
In order to be generic in its treatment, the proof will require the following regularity conditions on the performance measure
\begin{enumerate}
	\item $b_0 \neq 0$
	\item $\alpha - \Pf(\w)\cdot\gamma \geq 0$ for all $\w \in \W$
	\item $\beta - \Pf(\w)\cdot\delta \geq 0$ for all $\w \in \W$
	\item $-1 < f \leq \gamma\cdot P(\w) + \delta\cdot N(\w) \leq g$ for all $\w \in \W$
\end{enumerate}
Define $e_t := V(\w_{t+1},v_t) - v_t$. Then we can state the following lemmata which together yield the convergence bound proof.

\begin{lem}
$\frac{e_t}{1+f} \geq \Pf^\ast - v_t$
\end{lem}
\begin{proof}
Assume that for some $\w^\ast$, $\Pf(\w^\ast) = v_t + e_t + e'$ where $e' > 0$. Then we have
\begin{eqnarray*}
V(\w^\ast,v_t) &=& \br{\frac{e_t}{1+f} + e'}(1 + \gamma\cdot P(\w^\ast) + \delta\cdot N(\w^\ast)) - e_t\\
			   &\geq& \br{\frac{e_t}{1+f} + e'}(1 + f) - e_t\\
			   &=& e'(1 + f) > 0,
\end{eqnarray*}
which contradicts the fact that no classifier can achieve a valuation greater than $v_t + e_t$ at level $v_t$, thus proving the desired result.
\end{proof}

\begin{lem}
For any $\w$ that achieves $V(\w,v) = v + e$ such that $e \geq 0$, we have
\[
\Pf(\w) \geq v + \frac{e}{g+1}
\]
\end{lem}
\begin{proof}
Let $v' = v + \frac{e}{g+1}$. We will show that $V(\w,v') \geq v'$ which will establish the result by pseudo-linearity. We have
\begin{eqnarray*}
V(\w,v') - v' &=& c + (\alpha - v'\gamma)\cdot P(\w) + (\beta - v'\delta)\cdot N(\w) - v'\\
			  &=& c + (\alpha - v\gamma)\cdot P(\w) + (\beta - v\delta)\cdot N(\w) - v' - \frac{e}{g+1}(\gamma\cdot P(\w) + \delta\cdot N(\w))\\
			  &=& v + e - v' - \frac{e}{g+1}(\gamma\cdot P(\w) + \delta\cdot N(\w))\\
			  &\geq& v + e - v' - \frac{ge}{g+1} = 0,
\end{eqnarray*}
where we have used the bounds on $\gamma\cdot P(\w) + \delta\cdot N(\w)$ and the fact that $1 + g > 0$.
\end{proof}

Given the above results we can establish the convergence bound. More specifically, we can show the following: let $\Delta_t = \Pf^\ast - \Pf(\w_t)$. Then we have
\[
\Delta_{t+1} \leq \frac{g-f}{g+1}\cdot\Delta_t
\]
To see this, consider the following
\begin{eqnarray*}
\Delta_{t+1} &=& \Pf^\ast - \Pf(\w_{t+1}) \leq \Pf^\ast - \br{v_t + \frac{e_t}{g+1}} \leq \Pf^\ast - \br{v_t + \frac{(1+f)(\Pf^\ast - v_t)}{g+1}}\\
			 &=& \Pf^\ast - \br{\Pf(\w_t) + \frac{(1+f)(\Pf^\ast - \Pf(\w_t))}{g+1}} = \Delta_t - \frac{1+ f}{g + 1}\cdot\Delta_t = \frac{g-f}{g+1}\cdot\Delta_t,
\end{eqnarray*}
which proves the result. Notice that Table~\ref{tab:pseudolinear-perf-list} gives the rates of convergence for the different performance measures by calculating bounds on the value of $\frac{g-f}{g+1}$ for those performance measures.
\end{proof}

\section{An analysis of the \am Algorithm under Inexact Maximizations}
\label{app:noisy-am}
For this and the next section, we will, for the sake of simplicity, we will focus only on the F-measure for $\beta = 1$ and $p = 1/2$ so that $\theta = 1$. For this setting, the F-measure looks like the following: $F(P,N) = \frac{2P}{2 + P - N}$, and the valuation function looks like $V(\w,v) = (1-v/2)\cdot P(\w) + v/2\cdot N(\w)$. We shall denote the performance measure as $F(\w)$, and its optimal value as $F^\ast$. We will assume that the reward functions give bounded rewards in the range $[0,m)$.

So far we assumed that Step 4 in the Algorithm \am gave us $\w_{t+1}$ such that
\[
V(\w_{t+1},v_t) = \max_{\w\in\W}\ V(\w,v_t)
\]
Now we will only assume that $\w_{t+1}$ satisfies
\[
V(\w_{t+1},v_t) = \max_{\w\in\W}\ V(\w,v_t) - \epsilon_t
\]
We also assume that the level $v_t$ is only approximated in Step 5 of \am, i.e. using Lemma~\ref{lem:val-fn} we have
\[
v_t = F(\w_t) + \delta_t
\]
where $\delta_t$ is a signed real number.

Given these approximations, we can prove the following results
\begin{lem}
The following hold for the setting described above
\begin{enumerate}
	\item If $\delta_t \leq 0$ then $e_t \geq 0$
	\item If $\delta_t > 0$ then $e_t \geq -\delta_t\br{1 + \frac{m}{2}}$
	\item If $F^\ast < v_t$ (which can happen only if $\delta_t > 0$), then $e_t < 0$
	\item If $e_t  < 0$ then $F^\ast < v_t$
	\item We have
	\begin{enumerate}
		\item If $e_t \geq 0$, then $e_t \geq \br{\frac{2-m}{2}}(F^\ast - v_t)$.
		\item If $e_t < 0$, then $e_t \geq \br{\frac{2+m}{2}}(F^\ast - v_t)$.
	\end{enumerate}
	\item If $V(\w,v) = v + e$, then
	\begin{enumerate}
		\item If $e \geq 0$ then $F(\w) \geq v + \frac{2e}{2+m}$
		\item If $e < 0$ then $F(\w) \geq v + \frac{2e}{2-m}$
	\end{enumerate}
\end{enumerate}
\end{lem}
\begin{proof}
We give the proof in parts
\begin{enumerate}
	\item If $\delta_t \leq 0$ then this means that there exists a $\w$ such that $F(\w) \geq v_t$. The result then follows from pseudo linearity.
	\item $v_t = F(\w_t) + \delta_t$ gives us, by pseudo linearity of F-measure,
	\[
	(1 - v_t/2)\cdot P(\w_t) + v_t/2\cdot N(\w_t) = v_t - \delta_t\br{1 + \frac{P(\w_t) - N(\w_t)}{2}} \geq v_t - \delta_t\br{1 + \frac{m}{2}}.
	\]
	The bound on $e_t$ now follows from its definition.
	\item Suppose $e_t \geq 0$ then by pseudo linearity of F-measure, we have, for some $\w$, $V(\w,v_t) \geq v_t$ which means $F(\w) \geq v_t$ which contradicts the assumption.
	\item Suppose there exists $\w^\ast$ with $F(\w^\ast) = v_t + e'$ with $e' \geq 0$ then we have
	\[
	(1 - v_t/2)\cdot P(\w^\ast) + v_t/2\cdot N(\w^\ast) = v_t + e'\br{1 + \frac{P(\w^\ast) - N(\w^\ast)}{2}} \geq 0,
	\]
	which contradicts the fact that $e_t < 0$.
	\item Part (a) is simply Lemma 10. For part (b), we will prove that $F^\ast \leq v_t + \frac{2e_t}{2+m}$. Since $\frac{2}{2+m} > 0$, the result will follow. Assume the contrapositive that some $\w^\ast$ achieves $F(\w^\ast) = v_t + \frac{2e_t}{2+m} + e'$ for some $e' > 0$. Using the pseudo linearity of F-measure (and using the shorthand $v' = v_t + \frac{2e_t}{2+m} + e'$), this can be expressed as
\[
(1 - v'/2)\cdot P(\w^\ast) + v'/2\cdot N(\w^\ast) = v'
\]
where  for some $e' > 0$. Then we have
\begin{align*}
(1 - v_t/2)\cdot P(\w^\ast) + v_t/2\cdot N(\w^\ast) - v_t - e_t &= v' - v_t - e_t + \frac{1}{2}\br{\frac{2e_t}{2+m} + e'}(P(\w^\ast) - N(\w^\ast))\\
																	&= \frac{2e_t}{2+m} + e' - e_t + \frac{1}{2}\br{\frac{2e_t}{2+m} + e'}(P(\w^\ast) - N(\w^\ast))\\
																	&\geq \frac{2e_t}{2+m} + e' - e_t + \frac{m}{2}\br{\frac{2e_t}{2+m} + e'}\\
																	&= e'\br{1 + \frac{m}{2}} + e_t\br{\frac{2}{2+m} - 1 + \frac{m}{2+m}}\\
																	&= e'\br{1 + \frac{m}{2}} > 0,
\end{align*}
where we have assumed that $e'$ is chosen small enough so that $\frac{2e_t}{2+m} + e' < 0$ still and used the fact that $P(\w^\ast) - N(\w^\ast) \leq m$.
	\item Part (a) is simply Lemma 11. To prove part (b), we let $v' = v + \frac{2e}{2 - m}$, then we have 
\begin{eqnarray*}
(1 - \frac{v'}{2})\cdot P(\w) + \frac{v'}{2}\cdot N(\w) - v' &=& (1 - \frac{v}{2})\cdot P(\w) + \frac{v}{2}\cdot N(\w) - v' + \frac{e}{2-m}(N(\w) - P(\w))\\
														 &\geq& (1 - \frac{v}{2})\cdot P(\w) + \frac{v}{2}\cdot N(\w) - v' + \frac{me}{2-m}\\
														 &=& v + e - \br{v + \frac{2e}{2 - m}} + \frac{me}{2 - m}\\
														 &=& e\br{1 - \frac{2}{2 - m} + \frac{m}{2 - m}}\\
														 &=& 0,
\end{eqnarray*}
where the second inequality follows since $N(\w) - P(\w) \leq m$ and $e < 0$ by using the bounds on the reward functions. This proves the result.
\end{enumerate}
\end{proof}

\subsection{Convergence analysis}
We have the following cases with us
\begin{enumerate}
	\item Case 1 ($\delta_t \leq 0$): In this case we are setting $v_t$ to a value less than the F-measure of the current classifier. This should hurt performance - we know that $v_t = F(\w_t) + \delta_t$ which gives us, on applying part (a) of the previous lemma using $F^\ast - v_t = \Delta_t - \delta_t$, the following 
	\[
	e_t \geq \frac{2-m}{2}(\Delta_t - \delta_t).
	\]
	Note that we are guaranteed that $e_t \geq 0$ in this case. Now since the maximization in step 4 is also carried our approximately, we have $V(\w_{t+1},v_t) = v_t + e_t - \epsilon_t$. Now we have two sub cases
	\begin{enumerate}
		\item Case 1.1 ($\epsilon_t \leq e_t$): In this case we can apply part 6(a) of the previous lemma to get the following result
		\[
		\Delta_{t+1} \leq \frac{2m}{2+m}\Delta_t - \frac{2m}{2+m}\delta_t + \frac{2\epsilon_t}{2m}
		\]
		\item Case 1.2 ($\epsilon_t > e_t$): In this case we are actually making negative progress in the maximization step (since we have $V(\w_{t+1},v_t) \leq v_t$) and we can only invoke Lemma 5.6(b) to get
		\[
		\Delta_{t+1} \leq \frac{2\epsilon_t}{2-m}
		\]
		Note that the above result should not be interpreted as a one shot step to a very good classifier. The above result holds along with the condition that $\epsilon_t > e_t$. Thus the performance of the classifier is lower bounded by $e_t$ which depends on how far the current classifier is from the best.
	\end{enumerate}
	\item Case 2 ($\delta_t > 0$): In this case we are setting $v_t$ to the value higher than the F-measure of the current classifier. This can mislead the classifier and results in the following two sub-cases
	\begin{enumerate}
		\item Case 2.1 ($F^\ast \geq v_t$): In this case we are still setting $v_t$ to a legitimate value, i.e. one that is a valid F-measure for some classifier in the hypothesis class. This can only benefit the next optimization stage (in fact if we set $v_t = F^\ast$, then we would obtain the best classifier in this very iteration!). In this case $e_t \geq 0$ and we can use the analyses of Cases 1.1 and 1.2.
		\item Case 2.2 ($F^\ast < v_t$): In this case we are setting $v_t$ to an illegal value, one that is an unachievable value of F-measure. Consequently, using part 3 of the previous lemma, $e_t < 0$ and using part(b) of the previous lemma we get
		\[
		e_t \geq \frac{2+m}{2}(\Delta_t - \delta_t),
		\]
		which, upon applying part 6(b) of the previous lemma (since $e_t - \epsilon_t \leq e_t < 0$) will give us
		\begin{align*}
		\Delta_{t+1} &\leq \frac{2m}{2-m}(\delta_t - \Delta_t)+ \frac{2\epsilon_t}{2-m}\\
					 &\leq \frac{2m}{2-m}\delta_t + \frac{2\epsilon_t}{2-m}
		\end{align*}
	\end{enumerate}
\end{enumerate}

We can combine the cases together as follows
\begin{align*}
\Delta_{t+1} &\leq \max\bc{\1\bc{\delta \leq 0}\cdot\bc{\frac{2m}{2+m}\Delta_t - \frac{2m}{2+m}\delta_t + \frac{2\epsilon_t}{2+m}}, \1\bc{\epsilon_t > e_t}\cdot\frac{2\epsilon_t}{2-m}, \1\bc{\delta > 0}\cdot\bc{\frac{2m}{2-m}\delta_t + \frac{2\epsilon_t}{2-m}}}\\
			 &\leq \max\bc{\frac{2m}{2+m}\Delta_t + \frac{2m}{2+m}\abs{\delta_t} + \frac{2\epsilon_t}{2+m}, \1\bc{\epsilon_t > e_t}\cdot\frac{2\epsilon_t}{2-m}, \frac{2m}{2-m}\abs{\delta_t} + \frac{2\epsilon_t}{2-m}}\\
			 &\leq \frac{2m}{2+m}\Delta_t + \frac{2m}{2-m}\abs{\delta_t} + \frac{2\epsilon_t}{2-m}
\end{align*}

If we let $\eta = \frac{2m}{2+m}$, $\eta' = \frac{2m}{2-m}$, and $\xi_t = \abs{\delta_t} + \epsilon_t/m$, then this gives us
\[
\Delta_{t+1} \leq \eta\Delta_t + \eta'\xi_t,
\]
which gives us
\[
\Delta_T \leq \eta^T\Delta_0 + \frac{\eta'}{\eta}\cdot\sum_{i=0}^{T-1}\eta^{T-i}\xi_i
\]
This concludes our analysis.

\section{Proof of Theorem~\ref{thm:altmaxsgd-conv}}
\label{app:thm-altmaxsgd-conv-proof}
\begin{repthm}{thm:altmaxsgd-conv}
Let Algorithm~\ref{alg:amsgd} be executed with a performance measure $\Pf_{(\ba,\bb)}$ and reward functions with range $[0,m)$. Let $\eta = \eta(m)$ be the rate of convergence guaranteed for $\Pf_{(\ba,\bb)}$ by the \am algorithm. Set the epoch lengths to $s_e,s'_e = \softO{\frac{1}{\eta^{2e}}}$. Then after $e = \log_{\frac{1}{\eta}}\br{\frac{1}{\epsilon}\log^2\frac{1}{\epsilon}}$ epochs, we can ensure with probability at least $1-\delta$ that $\Pf^\ast - \Pf_{(\ba,\bb)}(\w_e) \leq \epsilon$. Moreover the number of samples consumed till this point is at most $\softO{\frac{1}{\epsilon^2}}$.
\end{repthm}
\begin{proof}
Using Hoeffding's inequality, standard regret and online-to-batch guarantees \cite{online-batch-single,zinkevich}, we can ensure that, if the stream lengths for the Model optimization stage and Challenge level estimation stage procedures are $s_e$ and $s_e'$ respectively, then for some fixed $c > 0$ that is independent of the stream length, we have
\[
\abs{\delta_t} \leq c\cdot\sqrt\frac{\log\frac{1}{\delta}}{s_e'}, \abs{\epsilon_t} \leq c\sqrt\frac{\log\frac{1}{\delta}}{s_e}
\]
Let $T = \log_{\frac{1}{\eta}}\br{\frac{1}{\epsilon}\log^2\frac{1}{\epsilon}}$ and $s_e = \br{\frac{2c}{m}}^2\br{\frac{1}{\eta}}^{2e}\log\frac{T}{\delta}$ and $s'_e = 4c^2\br{\frac{1}{\eta}}^{2e}\log\frac{T}{\delta}$ - this gives us, for each $e$, with probability at least $1 - \delta/T$,
\[
\xi_e \leq \eta^e
\]
Thus, using a union bound, with probability at least $1 - \delta$, we have, by the discussion in the previous section,
\begin{align*}
\Delta_T &\leq \eta^T\Delta_0 + \frac{\eta'}{\eta}\sum_{i=0}^{T-1}\eta^{T-i}\xi_i \leq \eta^T\Delta_0 + \frac{\eta'}{\eta}T\eta^T\\
		 &\leq \epsilon\Delta_0\log^{-2}\frac{1}{\epsilon} + \frac{\eta'}{\eta}\log_{\frac{1}{\eta}}\br{\frac{1}{\epsilon}\log^2\frac{1}{\epsilon}}\epsilon\log^{-2}\frac{1}{\epsilon}\\
		 &\leq \epsilon\br{\Delta_0 + \frac{\eta'}{\eta\log\frac{1}{\eta}}},
\end{align*}
where the last step follows from the fact that for any $\epsilon < 1/e^2$, we have
\[
\log\br{\frac{1}{\epsilon}\log^2\frac{1}{\epsilon}} \leq \log^{2}\frac{1}{\epsilon}
\]
Let $d = \br{\Delta_0 + \frac{\eta'}{\eta\log\frac{1}{\eta}}}$ so that we can later set $\epsilon' = \epsilon/d$, and $s = 4c^2\br{1 + \frac{1}{m^2}}$ so that $s_e + s'_e = s\br{\frac{1}{\eta}}^{2e}\log\frac{T}{\delta}$. The total number of samples required can then be calculated as
\[
\sum_{e=1}^Ts_e + s'_e = s\log\frac{T}{\delta}\sum_{e=1}^T\br{\frac{1}{\eta}}^{2e} = s\log\frac{T}{\delta}\frac{1}{1 - \eta^2}\br{\frac{1}{\eta^2}}^T \leq s\log\frac{T}{\delta}\frac{1}{1 - \eta^2}\frac{1}{\epsilon^2}\log^4\frac{1}{\epsilon}
\]
This gives the number of samples required as
\[
\O{\frac{1}{\epsilon^2}\log^4\frac{1}{\epsilon}\br{\log\log\frac{1}{\epsilon} + \log\frac{1}{\delta}}},
\]
to get an $\epsilon$-accurate solution with confidence $1-\delta$.
\end{proof}